\documentclass[twoside]{article}

\usepackage[accepted]{aistats2022}
\usepackage[round]{natbib}

\usepackage[utf8]{inputenc} 
\usepackage[T1]{fontenc}    
\usepackage{hyperref}       
\usepackage{url}            
\usepackage{booktabs}       
\usepackage{amsfonts}       
\usepackage{nicefrac}       
\usepackage{microtype}      
\usepackage{xcolor}         

\usepackage{bibentry}
\usepackage[utf8]{inputenc} 
\usepackage[T1]{fontenc}    

\usepackage{booktabs}       
\usepackage{amsfonts}       
\usepackage{nicefrac}       
\usepackage{microtype}      
\usepackage{xcolor}
\usepackage[normalem]{ulem}
\definecolor{mydarkblue}{rgb}{0,0.08,0.45}
\usepackage{hyperref}
\hypersetup{
    unicode=false,          
    pdftoolbar=true,        
    pdfmenubar=true,        
    pdffitwindow=false,     
    pdfstartview={FitH},    
    pdfnewwindow=true,      
    colorlinks=true,       
    linkcolor=mydarkblue,          
    citecolor=mydarkblue,        
    filecolor=magenta,      
    urlcolor=cyan           
}

\usepackage{algorithm}
\usepackage{amsmath,amssymb,amsthm,amsfonts,latexsym}        
\usepackage{mathtools}
\usepackage{graphicx}
\usepackage{bm}
\usepackage{pifont}
\usepackage{subfigure}
\usepackage{color}
\usepackage{appendix}
\usepackage{enumitem}
\usepackage{environ}         
\usepackage{etoolbox}        
\usepackage{setspace}
\usepackage[capitalise,nameinlink]{cleveref}
\usepackage{dsfont}
\usepackage[normalem]{ulem}
\usepackage{thmtools} 
\usepackage{thm-restate}
\usepackage{lipsum}
\usepackage{float}
\usepackage[algo2e]{algorithm2e} 
\usepackage{thm-restate}
\usepackage{mdframed}
\usepackage{titlesec}
\usepackage{xpatch}
\usepackage{placeins} 
\usepackage{wrapfig}

\newcommand{\nleft}{\mathclose\bgroup\left}
\newcommand{\nright}{\aftergroup\egroup\right}

\newcommand{\x}{w}
\newcommand{\xk}{w_{k}}

\newcommand{\inner}[2]{\langle #1,\, #2 \rangle}

\newcommand{\xkk}{w_{k+1}}

\newcommand{\xopt}{w^{*}}

\newcommand{\grad}[1]{\nabla f(#1)}

\newcommand{\proj}[2]{\Pi_{#1}\left( #2\right)}

\newcommand{\norm}[1]{\left\|#1\right\|}
\newcommand{\normsq}[1]{\left\|#1\right\|^{2}}
\newcommand{\Tr}[1]{\mathsf{Tr}\nleft(#1\nright)}
\newcommand{\E}{\mathbb{E}}

\newcommand{\fkopt}{f_{i_k}\kern-.65em{}^{\raisebox{1pt}{\tiny $*$}}\kern.05em}
\newcommand{\fj}{f_{i}}
\newcommand{\fit}{f_{i_t}}
\newcommand{\fjopt}{f_i^*}
\newcommand{\etak}{\eta_{k}}

\newcommand{\gradi}[1]{\nabla f_{i}(#1)}

\newcommand{\transpose}{^\mathsf{\scriptscriptstyle T}}

\newtheorem{lemma}{Lemma}
\newtheorem{proposition}{Proposition}

\newcommand{\aligns}[1]{\begin{align*} #1 \end{align*}}

\usepackage{amsmath}
\usepackage{amssymb}
\usepackage{amsthm}
\usepackage{mathrsfs}
\usepackage{amsfonts}
\usepackage{bm}
\usepackage{xparse}

\def\ceil#1{\lceil #1 \rceil}
\def\floor#1{\lfloor #1 \rfloor}
\def\1{\bm{1}}

\def\eps{{\epsilon}}

\DeclareMathAlphabet{\mathsfit}{\encodingdefault}{\sfdefault}{m}{sl}
\SetMathAlphabet{\mathsfit}{bold}{\encodingdefault}{\sfdefault}{bx}{n}

\newcommand{\R}{\mathbb{R}}

\newcommand{\lin}[1]{\langle#1\rangle}

\DeclareMathOperator{\diag}{diag}

\usepackage{thmtools, thm-restate}

\newtheorem{corollary}{Corollary}

\newcommand{\tightsub}[1]{{\kern -.1em \raise-.1em\hbox{\tiny$#1$}}{}}
\DeclareDocumentCommand{\Normal}{g}{\mathop{\mathcal{N}}\IfNoValueF{#1}{\nleft(#1\nright)}}
\DeclareDocumentCommand{\bigO}{g}{\mathop{\mathcal{O}}\IfNoValueF{#1}{\nleft(#1\nright)}}

\NewDocumentCommand{\paren}{sm}{%
  \IfBooleanTF{#1}{(#2)}{\nleft(#2\nright)}
}
\NewDocumentCommand{\braces}{sm}{%
  \IfBooleanTF{#1}{(#2)}{\nleft\{#2\nright\}}
}

\newlength{\dittowidth}

\def\adasvrg/{{AdaSVRG}}
\def\svrg/{{SVRG}}
\def\sarah/{{SARAH}}
\def\sag/{{SAG}}
\def\saga/{{SAGA}}
\def\adagrad/{{AdaGrad}}

\newcommand{\myquote}[1]{\null~\\{\null\hspace{.05\textwidth}\begin{minipage}[t]{.90\textwidth} #1 \end{minipage}}}

\mdfdefinestyle{mdframedthmbox}{%
	leftmargin=.0\textwidth,
	rightmargin=.0\textwidth,%
	innertopmargin=0.75em,
	innerleftmargin=.5em,
	innerrightmargin=.5em,
}
\newenvironment{thmbox}
	{%
		\begin{mdframed}[style=mdframedthmbox]%
	}{%
		\end{mdframed}%
	}

\newcommand{\dpr}[3][]{\left < #2,#3 \right >_{#1}}

\newcommand{\xtp}{x_{t+1}}

\newcommand{\xt}{x_{t}}
\newcommand{\f}{f^{*}}

\newcommand{\lr}{\eta}

\begin{document}

%
\runningtitle{SVRG meets AdaGrad}

%
\runningauthor{Dubois-Taine, Vaswani, Babanezhad, Schmidt, Lacoste-Julien}

\twocolumn[

\aistatstitle{SVRG meets AdaGrad: Painless Variance Reduction}

\aistatsauthor{Benjamin Dubois-Taine$^{\star 1}$ \And Sharan Vaswani$^{\star 2}$  \And Reza Babanezhad$^3$  \AND Mark Schmidt$^4$  \And Simon Lacoste-Julien$^5$ }
\aistatsaddress{$^1$ Universit\'e Paris-Saclay \And $^2$ Amii, University of Alberta \And $^3$ SAIT AI lab, Montreal \AND$^4$ University of British Columbia \And $^5$ Mila, Universit\'e de Montr\'eal} ]

\begin{abstract}
Variance reduction (VR) methods for finite-sum minimization typically require the knowledge of problem-dependent constants that are often unknown and difficult to estimate. To address this, we use ideas from adaptive gradient methods to propose \adasvrg/, which is a more-robust variant of \svrg/, a common VR method. \adasvrg/ uses \adagrad/ in the inner loop of \svrg/, making it robust to the choice of step-size. When minimizing a sum of $n$ smooth convex functions, we prove that a variant of \adasvrg/ requires $\tilde{O}(n + 1/\epsilon)$ gradient evaluations to achieve an $O(\epsilon)$-suboptimality, matching the typical rate, but without needing to know problem-dependent constants. Next, we leverage the properties of \adagrad/ to propose a heuristic that adaptively determines the length of each inner-loop in \adasvrg/. Via experiments on synthetic and  real-world datasets, we validate the robustness and effectiveness of \adasvrg/, demonstrating its superior performance over standard and other ``tune-free'' VR methods. 
\end{abstract}


\section{Introduction}
\label{sec:introduction}
Variance reduction (VR) methods~\citep{schmidt2017minimizing,konevcny2013semi,mairal2013optimization,shalev2013stochastic, johnson2013accelerating, mahdavi2013mixedgrad,konevcny2013semi, defazio2014saga, nguyen2017sarah} have proven to be an important class of algorithms for stochastic optimization. These methods take advantage of the finite-sum structure prevalent in machine learning problems, and have improved convergence over stochastic gradient descent (SGD) and its variants (see~\citep{gower2020variance} for a recent survey). For example, when minimizing a finite sum of $n$ strongly-convex, smooth functions with condition number $\kappa$, these methods typically require $O\left((\kappa + n) \, \log(1/\epsilon) \right)$ gradient evaluations to obtain an $\epsilon$-error. This improves upon the complexity of full-batch gradient descent (GD) that requires $O\left(\kappa n \, \log(1/\epsilon) \right)$ gradient evaluations, and SGD that has an $O(\kappa/\epsilon)$ complexity. Moreover, there have been numerous VR methods that employ Nesterov acceleration~\citep{allen2017katyusha, lan2019unified,song2020variance} and can achieve even faster rates.   

In order to guarantee convergence, VR methods require an easier-to-tune constant step-size, whereas SGD  needs a decreasing step-size schedule. Consequently, VR methods are commonly used in practice, especially when training convex models such as logistic regression or conditional Markov random fields~\citep{schmidt2015non}. However, all the above-mentioned VR methods require knowledge of the smoothness of the underlying function in order to set the step-size. The smoothness constant is often unknown and difficult to estimate in practice. Although we can obtain global upper-bounds on it for simple problems such as least squares regression, these bounds are usually too loose to be practically useful and result in sub-optimal performance. Consequently, implementing VR methods requires a computationally expensive search over a range of step-sizes. Furthermore, a constant step-size does not adapt to the function's local smoothness and may lead to poor empirical performance.   

Consequently, there have been a number of works that try to adapt the step-size in VR methods.~\citet{schmidt2017minimizing} and~\citet{mairal2013optimization} employ stochastic line-search procedures to set the step-size in VR algorithms. While they show promising empirical results using line-searches, these procedures have \emph{no theoretical convergence guarantees}. Recent works~\citep{tan2016barzilai, li2020almost} propose to use the Barzilai-Borwein (BB) step-size~\citep{barzilai1988two} in conjunction with two common VR algorithms - stochastic variance reduced gradient (\svrg/)~\citep{johnson2013accelerating} and the stochastic recursive gradient algorithm (\sarah/)~\citep{nguyen2017sarah}. Both~\citet{tan2016barzilai} and~\citet{li2020almost} can automatically set the step-size without requiring the knowledge of problem-dependent constants. However, in order to prove theoretical guarantees for strongly-convex functions, these techniques \emph{require the knowledge of both the smoothness and strong-convexity parameters.} In fact, their guarantees require using a small $O(1/\kappa^2)$ step-size, a highly-suboptimal choice in practice. Consequently, there is a gap in the theory and practice of adaptive VR methods. To address this, we make the following contributions.  
\subsection{Background and Contributions}
\label{sec:contributions}
\textbf{\svrg/ meets \adagrad/}: In~\cref{sec:adasvrg} we use \adagrad/~\citep{duchi2011adaptive,levy2018online}, an adaptive gradient method, with stochastic variance reduction techniques. We focus on \svrg/~\citep{johnson2013accelerating} and propose to use \adagrad/ within its inner-loop. We analyze the convergence of the resulting \adasvrg/ algorithm for minimizing convex functions (without strong-convexity). Using $O(n)$ inner-loops for every outer-loop (a typical setting used in practice~\citep{babanezhad2015stopwasting,sebbouh2019towards}), and \emph{any} bounded step-size, we prove that \adasvrg/ achieves an $\epsilon$-error (for $\epsilon = O(1/n)$) with $O(n/\epsilon)$ gradient evaluations (\cref{thm:adasvrg-main}). This rate matches that of \svrg/ with a constant step-size and $O(n)$ inner-loops~\citep[Corollary 10]{reddi2016stochastic}. However, unlike~\citet{reddi2016stochastic}, our result does not require knowledge of the smoothness constant in order to set the step-size. We note that other previous work~\citep{cutkosky2019momentum,liu2020adam} consider adaptive methods with variance reduction for non-convex minimization; however their algorithms still require knowledge of problem-dependent parameters.  

\textbf{Multi-stage \adasvrg/}: We propose a multi-stage variant of \adasvrg/ where each stage involves running \adasvrg/ for a fixed number of inner and outer-loops. In particular, multi-stage \adasvrg/ maintains a fixed-size outer-loop and doubles the length of the inner-loop across stages. We prove that it requires $O((n + 1/\epsilon) \log(1/\epsilon))$ gradient evaluations to reach an $O(\epsilon)$ error (\cref{thm:adasvrg-doubling}). This improves upon the complexity of decreasing step-size \svrg/ that requires $O(n + \sqrt{n}/\epsilon)$ gradient evaluations~\citep[Corollary 9]{reddi2016stochastic}; and matches the rate of \sarah/~\citep{nguyen2017sarah}.

\textbf{\adasvrg/ with adaptive termination}: Instead of using a complex multi-stage procedure, we prove that \adasvrg/ can also achieve the improved $O((n + 1/\epsilon) \log(1/\epsilon))$ gradient evaluation complexity by adaptively terminating its inner-loop (\cref{sec:adap-termination}). However, the adaptive termination requires the knowledge of problem-dependent constants, limiting its practical use. 

To address this, we use the favourable properties of \adagrad/ to design a practical heuristic for adaptively terminating the inner-loop. Our technique for adaptive termination is related to heuristics~\citep{pflug1983determination,yaida2018fluctuation,lang2019using,pesme2020convergence} that detect stalling for constant step-size SGD, and may be of independent interest. First, we show that when minimizing smooth convex losses, \adagrad/ has a two-phase behaviour - a first ``deterministic phase'' where the step-size remains approximately constant followed by a second ``stochastic'' phase where the step-size decreases at an $O(1/\sqrt{t})$ rate (\cref{thm:phase_trns}). We show that it is empirically possible to efficiently detect this phase transition and aim to terminate the \adasvrg/ inner-loop when \adagrad/ enters the stochastic phase. 

\textbf{Practical considerations and experimental evaluation}: In~\cref{sec:experiments}, we describe some of the practical considerations for implementing \adasvrg/ and the adaptive termination heuristic. We use standard real-world datasets to empirically verify the robustness and effectiveness of \adasvrg/. Across datasets, we demonstrate that \adasvrg/ consistently outperforms variants of \svrg/, \sarah/ and methods based on the BB step-size~\citep{tan2016barzilai, li2020almost}.

\textbf{Adaptivity to over-parameterization}:~\citet{defazio2018ineffectiveness} demonstrated the ineffectiveness of \svrg/ when training large over-parameterized models such as deep neural networks. We argue that this ineffectiveness can be partially explained by the interpolation property satisfied by over-parameterized models~\citep{schmidt2013fast, ma2018power, vaswani2019fast}. In the interpolation setting, SGD obtains an $O(1/\epsilon)$ gradient complexity when minimizing smooth convex functions~\citep{vaswani2019fast}, thus out-performing typical VR methods. However, interpolation is rarely exactly satisfied in practice, and using SGD can result in oscillations around the solution. On the other hand, although VR methods have a slower convergence, they do not oscillate, regardless of interpolation. In~\cref{app:interpolation}, we use \adagrad/ to exploit the (approximate) interpolation property, and employ the above heuristic to adaptively switch to \adasvrg/, thus avoiding oscillatory behaviour. We design synthetic problems controlling the extent of interpolation and show that the hybrid \adagrad/-\adasvrg/ algorithm can match or outperform both stochastic gradient and VR methods, thus achieving the best of both worlds.   
\section{Problem setup}
\label{sec:notation}
We consider the minimization of an objective $f:\R^d\rightarrow\R$ with a finite-sum structure, $\min_{w \in X} f(w)=\frac{1}{n}\sum_{i=1}^n f_i(w)$ where $X$ is a convex compact set of diameter $D$, meaning $\sup_{x, y \in X}\norm{x - y} \leq D$. Problems with this structure are prevalent in machine learning. For example, in supervised learning, $n$ represents the number of training examples, and $f_i$ is the loss function when classifying or regressing to training example~$i$. Throughout this paper, we assume~$f$ and each~$f_i$ are differentiable. We assume that~$f$ is convex, implying that there exists a solution~$\xopt \in X$ that minimizes it, and define~$f^*:= f(\xopt)$. Interestingly we \textit{do not need} each $f_i$ to be convex. We further assume that each function $\fj$ in the finite-sum is $L_i$-smooth, implying that $f$ is $L_{\max}$-smooth, where $L_{\max} = \max_{i} L_i$. We include the formal definitions of these properties in~\cref{app:definitions}.

We focus on the \svrg/ algorithm~\citep{johnson2013accelerating} since it is more memory efficient than alternatives like \sag/~\citep{schmidt2017minimizing} or \saga/~\citep{defazio2014saga}. \svrg/ has a nested inner-outer loop structure. In every outer-loop $k$, it computes the full gradient $\nabla f(\xk)$ at a \emph{snapshot} point $\x_k$. An outer-loop $k$ consists of $m_k$ inner-loops indexed by $t = 1, 2, \ldots m_k$ and the inner-loop iterate $x_1$ is initialized to $\x_k$. In outer-loop $k$ and inner-loop $t$, \svrg/ samples an example $i_t$ (typically uniformly at random) and takes a step in the direction of the \emph{variance-reduced gradient} $g_t$ using a constant step-size~$\eta$. This update can be expressed as:
\begin{align}
g_t & = \nabla \fit(x_t) - \nabla \fit(\xk) + \nabla f(\xk) \nonumber \\   
x_{t+1} & = \Pi_{X} [x_{t} - \eta \, g_t] \, ,
\end{align}
where $\Pi_{X}$ denotes the Euclidean projection onto the set $X$. The variance-reduced gradient is unbiased, meaning that $\E_{i_t } [g_t | x_t] = f(x_t)$. At the end of the inner-loop, the next snapshot point is typically set to either the last or averaged iterate in the inner-loop.

\svrg/ requires the knowledge of both the strong-convexity and smoothness constants in order to set the step-size and the number of inner-loops. These requirements were relaxed in~\citet{hofmann2015variance, kovalev2020don, gower2020variance} that only require knowledge of the smoothness. 

In order to set the step-size for \svrg/ without requiring knowledge of the smoothness, line-search techniques are an attractive option. Such techniques are a common approach to automatically set the step-size for (stochastic) gradient descent~\citep{armijo1966minimization, vaswani2019painless}. However, we show that an intuitive Armijo-like line-search to set the \svrg/ step-size is not guaranteed to converge to the solution. Specifically, we prove the following proposition in~\cref{app:line-search-counter-example}.

\begin{proposition}
\label{prop:line-search-counterexample}
If in each inner-loop $t$ of \svrg/, $\eta_t$ is set as the largest step-size satisfying the condition: $\eta_t \leq \eta_{\text{max}}$ and \begin{align*}
    f(x_{t} - \eta_t g_t) \leq f(x_t) - c \eta_t \norm{g_t}^2 \quad \text{where} \quad (c >0), 
\end{align*}
then for any $c > 0$, $\eta_{\text{max}} > 0$, there exists a 1-dimensional convex smooth function such that if $|x_t - \xopt|  \leq \min\{ \frac{1}{c}, 1\}$, then $|x_{t+1} - \xopt| \geq |x_t - \xopt|$, implying that the update moves the iterate away from the solution when it is close to it, preventing convergence.
\end{proposition}

In the next section, we suggest a novel approach using \adagrad/~\citep{duchi2011adaptive} to propose \adasvrg/, a provably-convergent VR method that is more robust to the choice of step-size. To justify our decision to use \adagrad/, we note that in general, there are (roughly) three common ways of designing methods that do not require knowledge of problem-dependent constants: (i) BB step-size, but it still requires knowledge of $L_{\text{max}}$ to guarantee convergence in the VR setting~\citep{tan2016barzilai, li2020almost}, (ii) Line-search methods that can fail to converge in the VR setting (\cref{prop:line-search-counterexample}), (iii) Adaptive gradient methods such as \adagrad/.  



 
.

\vspace{-5ex}
\section{Adaptive SVRG}
\label{sec:adasvrg}
\vspace{-1ex}
		\begin{algorithm}[h]
\caption{\adasvrg/ with fixed-sized inner-loop}
\label{alg:adasvrg}
\textbf{Input}: $w_0$ (initial point), $K$ (outer-loops), $m$ (inner-loops) \\
\For{$k \leftarrow 0$  \KwTo  $K-1$}{
    Compute full gradient $\nabla f(w_k)$ \\
    $\eta_k \leftarrow$ Compute step-size \\
    Initialize: $x_1 = w_k$ and $G_0 = 0$. \\
    \For{$t \leftarrow 1$   \KwTo  $m$}{
    Sample $i_t \sim \text{Uniform}\{1, 2, \ldots n\}$ \\
    $g_t = \nabla \fit(x_t) - \nabla \fit(w_k) + \nabla f(w_k)$\\
    $G_t = G_{t-1} + \normsq{g_t}$ \\
    $A_t = G_t^{1/2}$\\
    $x_{t+1} = \proj{X, A_t}{x_t - \eta_{k} A_t^{-1}g_t}$
    }
    $w_{k+1} = \frac{1}{m} \sum_{t=1}^{m} x_t$
}
Return $\bar{w}_K = \frac{1}{K} \sum_{k = 1}^{K} w_{k}$
\end{algorithm}

\begin{algorithm}[h]
\caption{Multi-stage \adasvrg/}
\label{alg:ms-adasvrg}
\textbf{Input}: $\bar{w}^0$ (initial point), $K$ (outer-loops), $\epsilon$ (target error) \\
Initialize $I = \log(1/\epsilon)$ \\
\For{$i \leftarrow 1$  \KwTo  $I$}{
    $m^i = 2^{i+1}$ \\
    $\bar{w}^{i}$ = \cref{alg:adasvrg}($\bar{w}^{i-1}$, $K$, $m^i$)
}
Return $\bar{w}^{I}$
\end{algorithm}

Like SVRG, \adasvrg/ has a nested inner-outer loop structure  and relies on computing the full gradient in every outer-loop. However, it uses \adagrad/ in the inner-loop, using the variance reduced gradient $g_t$ to update the preconditioner $A_t$ in the inner-loop $t$. \adasvrg/ computes the step-size $\eta_k$ in every outer-loop (see~\cref{sec:experiments} for details) and uses a preconditioned variance-reduced gradient step to update the inner-loop iterates: $$x_{t+1} = \proj{X, A_t}{x_t - \eta_{k} A_t^{-1} g_t}.$$ Here, $\proj{X, A}{\cdot}$ is the projection onto set $X$ with respect to the norm induced by a symmetric positive definite matrix $A$ (such projections are common to adaptive gradient methods~\citep{duchi2011adaptive, levy2018online, reddi2018convergence}).  \adasvrg/ then sets the next snapshot $\xkk$ to be the average of the inner-loop iterates. Throughout the main paper, we will only focus on the scalar variant~\citep{ward2018adagrad} of \adagrad/ (see~\cref{alg:adasvrg} for the pseudo-code). We defer the general diagonal and matrix variants (see~\cref{app:algorithm-general-case} for the pseudo-code) and their corresponding theory to the Appendix.

We now analyze the convergence of \adasvrg/. We start with the analysis of a single outer-loop, and prove the following lemma in~\cref{app:lemma:adasvrg-outerloop}
\begin{restatable}[\adasvrg/ with single outer-loop]{lemma}{restateAdaSVRGouterloop}
Assume (i) convexity of $f$, (ii) $L_{\max}$-smoothness of $f_i$ and (iii) bounded feasible set with diameter $D$. Defining $\rho := \big( \frac{D^2}{\etak} + 2\etak \big) \sqrt{L_{\text{max}}}$, for any outer loop $k$ of \adasvrg/, with (a) inner-loop length $m_k$ and (b) step-size $\etak$, 
\begin{align*}
\E[f(w_{k+1}) - f^*] \leq \frac{\rho^2}{m_k} + \frac{\rho\sqrt{\E[f(w_k) - f^*]}}{\sqrt{m_k}}
\end{align*}
\label{lemma:adasvrg-outerloop}
\end{restatable}
\vspace{-3ex}
The proof of the above lemma leverages the theoretical results of \adagrad/~\citep{duchi2011adaptive, levy2018online}. Specifically, the standard \adagrad/ analysis bounds the ``noise'' term by the variance in the stochastic gradients. On the other hand, we use the properties of the variance reduced gradient in order to upper-bound the noise in terms of the function suboptimality.

\cref{lemma:adasvrg-outerloop} shows that a single outer-loop of \adasvrg/ converges to the minimizer as $O(1/\sqrt{m})$, where $m$ is the number of inner-loops. This implies that in order to obtain an $\epsilon$-error, a single outer-loop of  \adasvrg/ requires $O(n + 1/\epsilon^2)$ gradient evaluations. This result holds for \emph{any} bounded step-size and requires setting $m = O(1/\epsilon^2)$. This  ``single outer-loop convergence'' property of  \adasvrg/ is unlike  \svrg/ or any of its variants; running only a single-loop of \svrg/ is ineffective, as it stops making progress at some point, resulting in the iterates oscillating in a neighbourhood of the solution. The favourable behaviour of \adasvrg/ is similar to \sarah/, but unlike \sarah/, the above result does not require computing a recursive gradient or knowing the smoothness constant.   

Next, we consider the convergence of \adasvrg/ with a fixed-size inner-loop and multiple outer-loops. In the following theorems, we assume that we have a bounded range of step-sizes implying that for all $k$, $\etak \in [\eta_{\text{min}}, \eta_{\text{max}}]$. For brevity, similar to~\cref{lemma:adasvrg-outerloop}, we define $\rho \coloneqq \big( \frac{D^2}{\eta_{\text{min}}} + 2\eta_{\text{max}}\big) \sqrt{L_{\text{max}}}$.
\begin{restatable}[\adasvrg/ with fixed-size inner-loop]{theorem}{restateAdaSVRGmain}
\label{thm:adasvrg-main}
Under the same assumptions as~\cref{lemma:adasvrg-outerloop}, \adasvrg/ with (a) step-sizes $\etak \in [\eta_{\text{min}}, \eta_{\text{max}}]$, (b) inner-loop size $m_k = n$ for all $k$, results in the following convergence rate after $K \leq n$ iterations.
\begin{align*}
    \E[f(\bar{w}_K) - f^*] \leq \frac{\rho^2 ( 1+ \sqrt{5}) + \rho\sqrt{2\left(f(w_0) - f^*\right)}}{K}
\end{align*}
where $\bar{w}_K = \frac{1}{K} \sum_{k = 1}^{K} w_{k}$. 
\end{restatable}
The proof (refer to~\cref{app:thm:adasvrg-main})  recursively uses the result of~\cref{lemma:adasvrg-outerloop} for $K$ outer-loops.  

The above result requires a fixed inner-loop size $m_k = n$, a setting typically used in practice~\citep{babanezhad2015stopwasting, gower2020variance}. Notice that the above result holds only when $K \leq n$. Since $K$ is the number of outer-loops, it is typically much smaller than $n$, the number of functions in the finite sum, justifying the theorem's $K \leq n$ requirement. Moreover, in the sense of generalization error, it is not necessary to optimize below an $O(1/n)$ accuracy~\citep{boucheron2005theory, sridharan2008fast}. 

\cref{thm:adasvrg-main} implies that \adasvrg/ can reach an $\epsilon$-error (for $\epsilon = \Omega(1/n)$) using $\mathcal O(n/\epsilon)$ gradient evaluations. This result matches the complexity of constant step-size \svrg/ (with $m_k = n$) of~\citep[Corollary 10]{reddi2016stochastic} but without requiring the knowledge of the smoothness constant. However, unlike \svrg/ and \sarah/, the convergence rate depends on the diameter $D$ rather than $\norm{\x_0 - \xopt}$, the initial distance to the solution. This dependence arises due to the use of \adagrad/ in the inner-loop, and is necessary for adaptive gradient methods. Specifically,~\citet{cutkosky2017online} prove that any adaptive (to problem-dependent constants) method will necessarily incur such a dependence on the diameter. Hence, such a diameter dependence can be considered to be the ``cost'' of the lack of knowledge of problem-dependent constants. 

Since the above result only holds for $\epsilon = \Omega(1/n)$, we propose a multi-stage variant (\cref{alg:ms-adasvrg}) of \adasvrg/ that requires $O((n+1/\epsilon)\log(1/\epsilon))$ gradient evaluations to attain an $O(\epsilon)$-error for any $\epsilon$. To reach a target suboptimality of $\epsilon$, we consider $I = \log(1/\epsilon)$ stages. For each stage $i$, \cref{alg:ms-adasvrg} uses a fixed number of outer-loops $K$ and inner-loops $m^i$ with stage $i$ is initialized to the output of the $(i-1)$-th stage. In~\cref{app:thm:adasvrg-doubling}, we prove the following rate for multi-stage \adasvrg/. 
\begin{restatable}[Multi-stage \adasvrg/]{theorem}{restateAdaSVRGdoubling}
Under the same assumptions as~\cref{thm:adasvrg-main}, multi-stage \adasvrg/ with $I = \log(1/\epsilon)$ stages, $K \geq 3$ outer-loops and $m^i = 2^{i+1}$ inner-loops at stage $i$, requires $O(n\log\frac{1}{\epsilon} + \frac{1}{\epsilon})$ gradient evaluations to reach a $\left(\rho^2 (1+\sqrt{5}) \right) \epsilon$-sub-optimality.
\label{thm:adasvrg-doubling}
\end{restatable}

We see that multi-stage \adasvrg/ matches the convergence rate of \sarah/ (upto constants), but does so without requiring the knowledge of the smoothness constant to set the step-size. Observe that the number of inner-loops increases with the stage i.e. $m^i = 2^{i+1}$. The intuition behind this is that the convergence of \adagrad/ (used in the $k$-th inner-loop of \adasvrg/) is slowed down by a ``noise'' term proportional to $f(w_k) - f^*$ (see~\cref{lemma:adasvrg-outerloop}). When this ``noise'' term is large in the earlier stages of multi-stage \adasvrg/, the inner-loops have to be short in order to maintain the overall $O(1/\eps)$ convergence. However, as the stages progress and the suboptimality decreases, the ``noise'' term becomes smaller, and the algorithm can use longer inner-loops, which reduces the number of full gradient computations, resulting in the desired convergence rate. 

Thus far, we have focused on using \adasvrg/ with fixed-size inner-loops. Next, we consider variants that can adaptively determine the inner-loop size.








\section{Adaptive termination of inner-loop}
\label{sec:adap-termination}
Recall that the convergence of a single outer-loop $k$ of \adasvrg/ (\cref{lemma:adasvrg-outerloop}) is slowed down by the $\sqrt{\nicefrac{\left(f(w_k) - f^*\right)}{m_k}}$ term. Similar to the multi-stage variant, the suboptimality $f(w_k) - f^*$ decreases as \adasvrg/ progresses. This allows the use of longer inner-loops as $k$ increases, resulting in fewer full-gradient evaluations. We instantiate this idea by setting $m_k = O\left(\nicefrac{1}{\left(f(w_k) - f^*\right)}\right)$. Since this choice requires the knowledge of $f(w_k) - f^*$, we alternatively consider using $m_k = O(1/\epsilon)$, where $\epsilon$ is the desired sub-optimality. We prove the following theorem in~\cref{app:thm:adasvrg-at}.
\begin{restatable}[\adasvrg/ with adaptive-sized inner-loops]{theorem}{restateAdaSVRGadaptivetermination}
Under the same assumptions as~\cref{lemma:adasvrg-outerloop}, \adasvrg/ with (a) step-sizes $\etak \in [\eta_{\text{min}}, \eta_{\text{max}}]$, (b1) inner-loop size $m_k = \frac{4\rho^2}{\epsilon}$ for all $k$ or (b2) inner-loop size $m_k = \frac{4\rho^2}{f(w_k) - f^*}$ for outer-loop $k$, results in the following convergence rate,
\begin{align*}
\E[f(w_K) - f^*] \leq (3/4)^K [f(w_0) - f^*]. 
\end{align*}
\label{thm:adasvrg-at}
\end{restatable}
\vspace{-4ex}
The above result implies a linear convergence in the number of outer-loops, but each outer-loop requires $O(1/\epsilon)$ inner-loops. Hence,~\cref{thm:adasvrg-at} implies that \adasvrg/ with adaptive-sized inner-loops requires $O\left((n + 1/\epsilon) \log(1/\epsilon) \right)$ gradient evaluations to reach an $\epsilon$-error. This improves upon the rate of \svrg/ and matches the convergence rate of \sarah/ that also requires inner-loops of length $O(1/\epsilon)$. Compared to~\cref{thm:adasvrg-main} that has an average iterate convergence (for $\bar{w}_K$),~\cref{thm:adasvrg-at} has the desired convergence for the last outer-loop iterate $w_{K}$ and also holds for any bounded sequence of step-sizes. However, unlike~\cref{thm:adasvrg-main}, this result (with either setting of $m_k$) requires the knowledge of problem-dependent constants in $\rho$. 

To address this issue, we design a heuristic for adaptive termination in the next sections. We start by describing the two phase behaviour of \adagrad/ and subsequently utilize it for adaptive termination in \adasvrg/.  

\subsection{Two phase behaviour of \adagrad/}
\label{sec:two-phase-adagrad}
Diagnostic tests~\citep{pflug1983determination, yaida2018fluctuation,lang2019using,pesme2020convergence} study the behaviour of the SGD dynamics to automatically control its step-size. Similarly, designing the adaptive termination test requires characterizing the behaviour of \adagrad/ used in the inner loop of \adasvrg/. 

We first investigate the dynamics of constant step-size \adagrad/ in the stochastic setting. Specifically, we monitor the evolution of $\|G_t\|_* = \sqrt{\Tr{G_t}}$ across iterations. We define $\sigma^2 \coloneqq \sup_{x \in X} \mathbb E_{i} \|\nabla f_i(x) - \nabla f(x) \|^2$ as a uniform upper-bound on the variance in the stochastic gradients for all iterates. We prove the following theorem showing that there exists an iteration $T_0$ when the evolution of $\|G_t\|_*$ undergoes a phase transition. 
\begin{restatable}[Phase Transition in \adagrad/ Dynamics]{theorem}{restateAdaGradPhaseTrans}
\label{thm:phase_trns}
Under the same assumptions as~\cref{lemma:adasvrg-outerloop} and (iv) $\sigma^2$-bounded stochastic gradient variance and defining $T_0 = \frac{\rho^2 L_{\text{max}}}{\sigma^2}$, for constant step-size \adagrad/ we have 
$\E \|G_t\|_{*} = O(1)  \text{ for } t \leq T_0$, and $\E \|G_t\|_{*} = O (\sqrt{t-T_0})  \text{ for } t \geq T_0$. 
\end{restatable}
Theorem~\ref{thm:phase_trns} (proved in~\cref{app:thm:phase_trns}) indicates that the norm $\|G_t\|_*$ is bounded by a constant for all $t \leq T_0$, implying that its rate of growth is slower than $\log(t)$. This implies that the step-size of \adagrad/ is approximately constant (similar to gradient descent in the full-batch setting) in this first phase until iteration $T_0$. Indeed, if $\sigma = 0$, $T_0 = \infty$ and \adagrad/ is always in this deterministic phase. This result generalizes~\citep[Theorem 3.1]{qian2019implicit} that analyzes the diagonal variant of \adagrad/ in the deterministic setting. After iteration $T_0$, the noise $\sigma^2$ starts to dominate, and \adagrad/ transitions into the stochastic phase where $\|G_t\|_*$ grows as $O(\sqrt{t})$. In this phase, the step-size decreases as $O(1/\sqrt{t})$, resulting in slower convergence to the minimizer. \adagrad/ thus results in an overall $O \left(\nicefrac{1}{T} + \nicefrac{\sigma^2}{\sqrt{T}} \right)$ rate~\citep{levy2018online}, where the first term corresponds to the deterministic phase and the second to the stochastic phase.   

Using the fact that for \adagrad/ in the inner-loop of \adasvrg/, the term $f(w_k) - f^*$ behaves as the ``noise'' similar to $\sigma^2$, we prove~\cref{cor:phase_trns} that shows that the same phase transition as~\cref{thm:phase_trns} happens at iteration $T_0 := \frac{\rho^2 L_{\text{max}} }{f(w_k) - \f}$. By detecting this phase transition, \adasvrg/ can terminate its inner-loop after $T_0$ iterations. This is exactly the behaviour we want to replicate in order to avoid the slowdown due to the $\sqrt{f(w_k) - f^*}$ term and is similar to the ideal adaptive termination required by~\cref{thm:adasvrg-at}. Putting these results together, we conclude that if \adasvrg/ terminates each inner-loop according to a diagnostic test that can \emph{exactly} detect the phase transition in the growth of $\|G_T\|_*$, the resulting \adasvrg/ variant will have an $\mathcal O((n+\frac{1}{\epsilon}) \log(\frac{1}{\epsilon}))$ gradient complexity. 

Since the exact detection of this phase transition is not possible, we design a heuristic to detect it \emph{without} requiring the knowledge of problem-dependent constants. 
\subsection{Heuristic for adaptive termination}
\label{sec:adap-term-test}
\begin{algorithm}[h]
\caption{\adasvrg/ with adaptive termination test}
\label{alg:adasvrg-at-test}
\textbf{Input}: $w_0$ (initial point), $K$ (outer-loops), $\theta$ (adaptive termination parameter), $M$ (maximum inner-loops) \\
\For{$k \leftarrow 0$  \KwTo  $K-1$}{
    Compute full gradient $\nabla f(w_k)$ \\
    $\eta_k \leftarrow$ Compute step-size \\
    Initialize \adagrad/: $x_1 = w_k$ and $G_0 = 0$. \\
    \For{$t \leftarrow 1$ \KwTo $M$}{
    Sample $i_t \sim \text{Uniform}\{1, 2, \ldots n\}$ \\
    $g_t = \nabla \fit(x_t) - \nabla \fit(w_k) + \nabla f(w_k)$\\
    $G_t = G_{t-1} + \normsq{g_t}$ \\
    $A_t = G_t^{1/2}$\\
    \If{$t$ mod $2$ $==0$ and $t \geq n$}
    {$R = \frac{\|G_{t}\|^2_* - \|G_{t/2}\|^2_*}{\|G_{t/2}\|^2_*}$\\
    \If{$R \geq \theta$}{
    Terminate inner loop
    }
    }
    $x_{t+1} = \proj{X, A_t}{x_t - \eta_{k} A_t^{-1}g_t}$
    }
    $w_{k+1} = \frac{1}{m_k} \sum_{t=1}^{m_k} x_t$
}
Return $w_{K}$
\end{algorithm}
Similar to tests used to detect stalling for SGD~\citep{pflug1983determination, pesme2020convergence}, the proposed diagnostic test has a burn-in phase of $\nicefrac{n}{2}$ inner-loop iterations that allows the initial \adagrad/ dynamics to stabilize. After this burn-in phase, for every even iteration, we compute the ratio $R = \frac{\|G_{t}\|^2_* - \|G_{t/2}\|^2_*}{\|G_{t/2}\|^2_*}$. Given a threshold hyper-parameter $\theta$, the test terminates the inner-loop when $R \geq \theta$. In the first deterministic phase, since the growth of $\|G_{t}\|^2_*$ is slow,  $\|G_{2t}\|^2_* \approx \|G_{t}\|^2_*$ and $R \approx 0$. In the stochastic phase, $\|G_{t}\|^2_* = O(t)$, and $R \approx 1$, justifying that the test can distinguish between the two phases. \adasvrg/ with this test is fully specified in~\cref{alg:adasvrg-at-test}. Experimentally, we use $\theta = 0.5$ to give an early indication of the phase transition.\footnote{We note that \sarah/~\citep{nguyen2017sarah} also suggests a heuristic for adaptively terminating the inner-loop. However, their test is not backed by any theoretical insight.} 

\section{Experiments}
\label{sec:experiments}
\begin{figure*}[h!]
\centering    
\subfigure{\label{fig:main-logistic-3datasets-bs=64}\includegraphics[scale = 0.29]{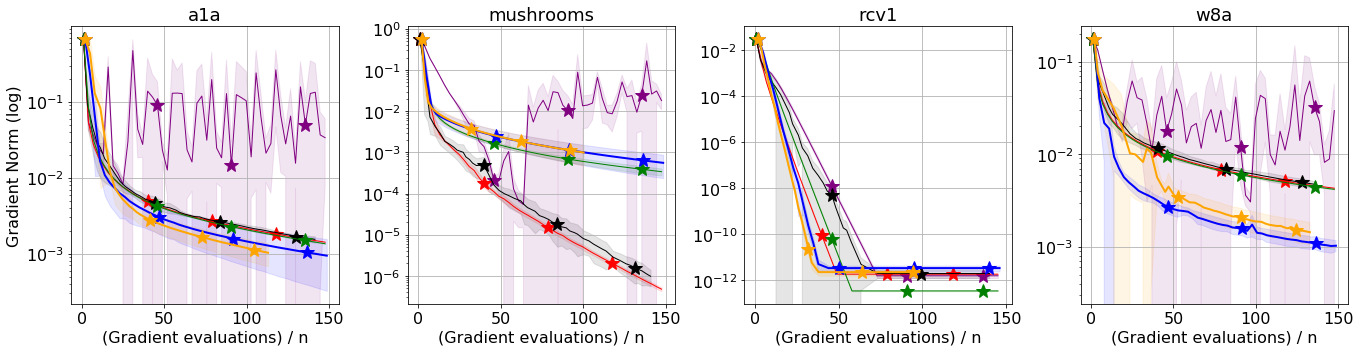}}\\
\addtocounter{subfigure}{-1}
\subfigure[Logistic loss]{\label{fig:main-stepsize-logistic-3datasets-bs=64}\includegraphics[scale = 0.29]{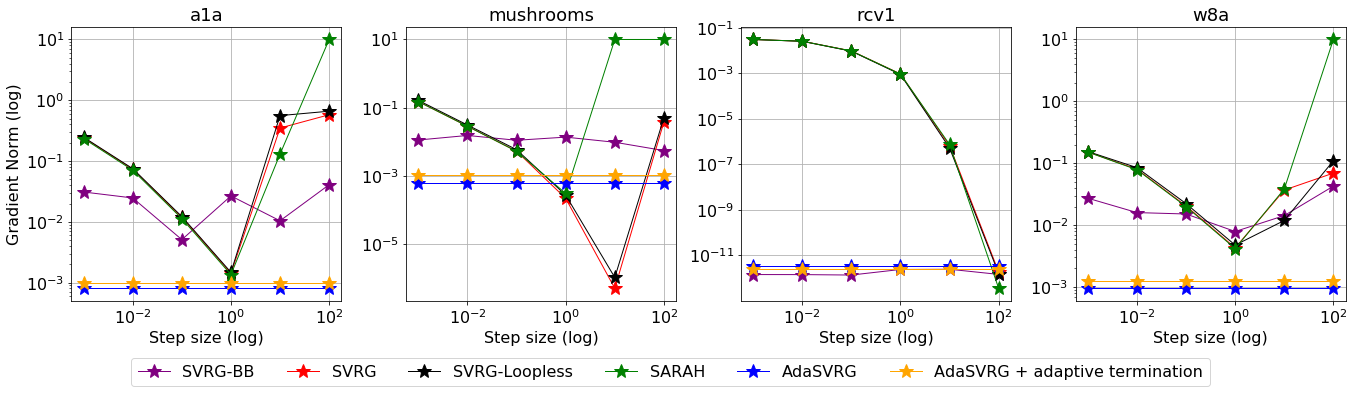}}
\subfigure{\label{fig:main-huber-3datasets-bs=64}\includegraphics[scale = 0.29]{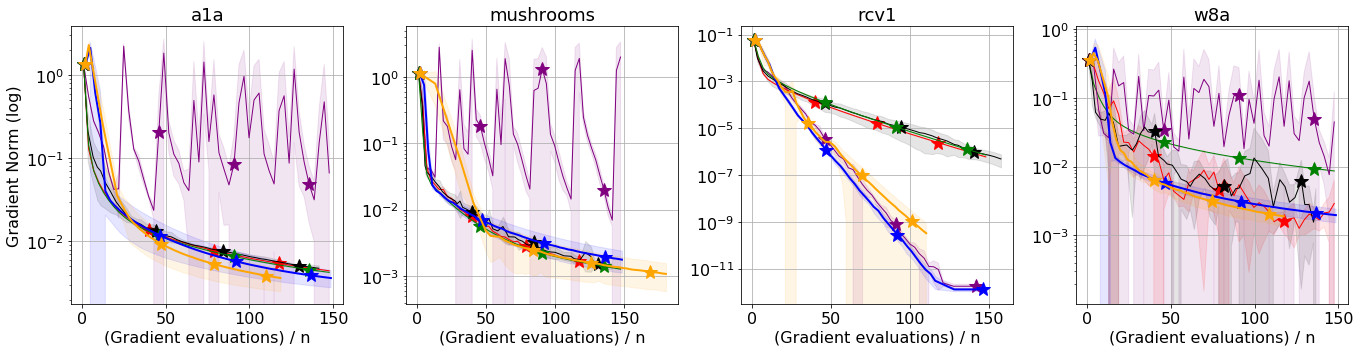}} \\
 \addtocounter{subfigure}{-1}
\subfigure[Huber loss]{\label{fig:main-stepsize-huber-3datasets-bs=64}\includegraphics[scale = 0.29]{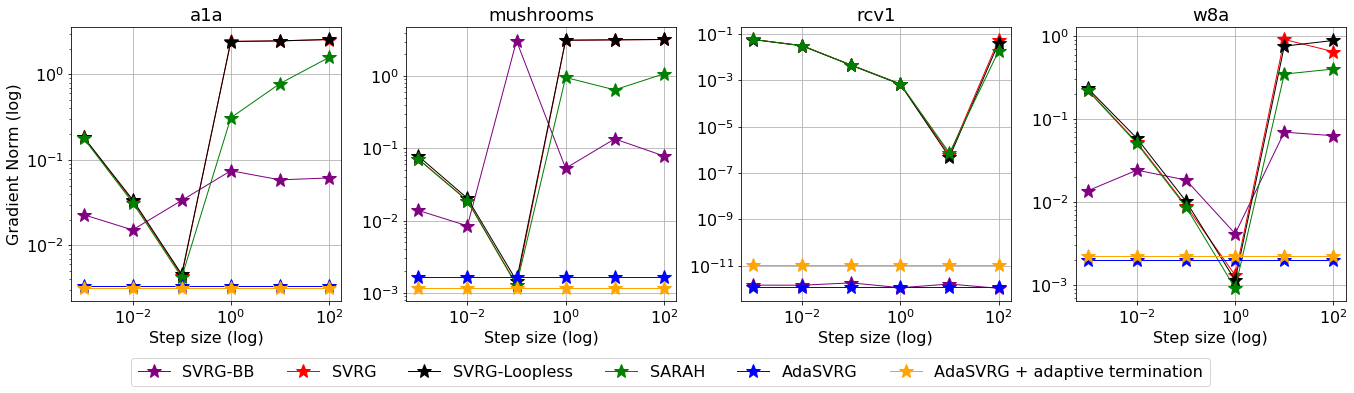}}
\vspace{-2ex}
\caption{Comparison of \adasvrg/ against \svrg/ variants, \svrg/-BB and \sarah/ with batch-size = $64$ for logistic loss (top 2 rows) and Huber loss (bottom 2 rows). For both losses, we compare \adasvrg/ against the best-tuned variants, and show the sensitivity to step-size (we limit the gradient norm to a maximum value of 10).}
\end{figure*}

\begin{figure*}[h!]
\centering     
\subfigure{\label{fig:4datasets-squared-epoch-bs=64}\includegraphics[scale = 0.3]{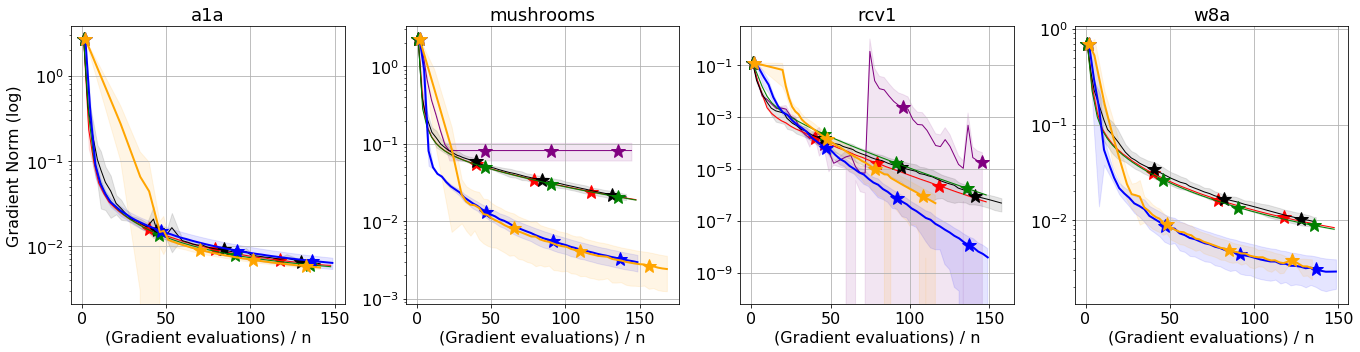}} \\
 \addtocounter{subfigure}{-1}
\subfigure[Squared loss]{\label{fig:4datasets-squared-stepsize-bs=64}\includegraphics[scale = 0.3]{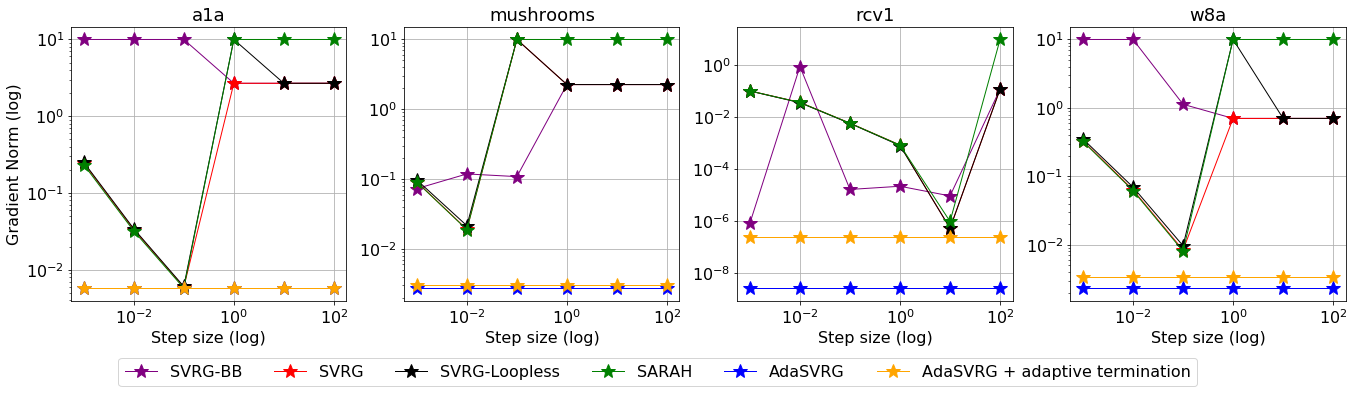}}
\vspace{-2ex}
\caption{Comparison of \adasvrg/ against \svrg/ variants, \svrg/-BB and \sarah/ with batch-size = $64$ for squared loss We compare \adasvrg/ against the best-tuned variants, and show the sensitivity to step-size (we limit the gradient norm to a maximum value of 10). In cases where SVRG-BB diverged, we remove these curves.}
\end{figure*}
We first describe the practical considerations for implementing \adasvrg/ and then evaluate its performance on real and synthetic datasets. We do not use projections in our experiments as these problems have an unconstrained $w^*$ with finite norm (we thus assume $D$ is big enough to include it), and that we empirically observed that our iterates always stayed bounded, thus not requiring any projection.\footnote{We note in passing that the literature for unconstrained stochastic optimization often explicitly assumes that the iterates stay bounded~\citep{ahn2020sgd, bollapragada2019exact, babanezhad2015stopwasting}.} 

\textbf{Implementing \adasvrg/}: Though our theoretical results hold for any bounded sequence of step-sizes, its choice affects the practical performance of \adagrad/~\citep{vaswani2020adaptive} (and hence \adasvrg/). Theoretically, the optimal step-size minimizing the bound in~\cref{lemma:adasvrg-outerloop} is given by $\eta^* = \frac{D}{\sqrt{2}}$. Since we do not have access to $D$, we use the following heuristic to set the step-size for each outer-loop of \adasvrg/. In outer-loop $k$, we approximate $D$ by $\|\xk - \xopt\|$, that can be bounded using the co-coercivity of smooth convex functions as $\| \xk - \xopt\| \geq \nicefrac{1}{L_{\text{max}}} \| \nabla f(w_k) \|$~\citep[Thm.~2.1.5 (2.1.8)]{nesterov2013introductory}. We have access to $\nabla f(w_k)$ for the current outer-loop, and store the value of $\nabla f(w_{k-1})$ in order to approximate the smoothness constant. Specifically, by co-coercivity, $L_{\text{max}} \geq L_k \coloneqq \frac{\| \nabla f(w_k) - \nabla f(w_{k-1})\|}{\| w_k - w_{k-1}\|}$. Putting these together, $\etak = \frac{\| \nabla f(w_k)\|}{\sqrt{2} \, \max_{i=0, \dots, k} L_i}$\footnote{For $k = 0$, we compute the full gradient at a random point $w_{-1}$ and approximate $L_0$ in the same way.}. Although a similar heuristic could be used to estimate $L_{\text{max}}$ for \svrg/ or \sarah/, the resulting step-size is larger than $1/L_{\text{max}}$ implying that it would not have any theoretical guarantee, while our results hold for any bounded sequence of step-sizes. Although~\cref{alg:adasvrg} requires setting $\xkk$ to be the average of the inner-loop iterates, we use the last-iterate and set $\xk = x_{m_k}$, as this is a more common choice~\citep{johnson2013accelerating, tan2016barzilai} and results in better empirical performance. We compare two variants of \adasvrg/, with (i) fixed-size inner-loop~\cref{alg:adasvrg} and (ii) adaptive termination~\cref{alg:adasvrg-at-test}. We handle a general batch-size $b$, and set $m = \nicefrac{n}{b}$ for~\cref{alg:adasvrg}. This is a common practical choice~\citep{babanezhad2015stopwasting, gower2020variance, kovalev2020don}. For~\cref{alg:adasvrg-at-test}, the burn-in phase consists of $\nicefrac{n}{2b}$ iterations and $M = \nicefrac{10n}{b}$. 

\textbf{Evaluating \adasvrg/}: In order to assess the effectiveness of \adasvrg/, we experiment with binary classification on standard LIBSVM datasets~\citep{libsvm}. In particular, we consider $\ell_2$-regularized problems (with regularization set to $\nicefrac{1}{n}$) with three losses - logistic loss, the squared loss or the Huber loss. For each experiment we plot the median and standard deviation across 5 independent
runs. In the main paper, we show the results for four of the datasets and relegate the results for the three others to~\cref{app:additional-experiments}. Similarly, we consider batch-sizes in the range $[1, 8, 64, 128]$, but only show the results for $b = 64$ in the main paper. 

We compare the \adasvrg/ variants against \svrg/~\citep{johnson2013accelerating}, loopless-\svrg/~\citep{kovalev2020don}, \sarah/~\citep{nguyen2017sarah}, and \svrg/-BB~\citep{tan2016barzilai}, the only other tune-free VR method.\footnote{We do not compare against \sag/~\citep{schmidt2017minimizing} because of its large memory footprint.} Since each of these methods requires a step-size, we search over the grid $[10^{-3}, 10^{-2}, 10^{-1}, 1, 10, 100]$, and select the best step-size for each algorithm and each experiment. As is common, we set $m = \nicefrac{n}{b}$ for each of these methods. We note that though the theoretical results of \svrg/-BB require a small $O(1/\kappa^2)$ step-size and $O(\kappa^2)$ inner-loops,~\citet{tan2016barzilai} recommends setting $m = O(n)$ in practice. Since \adagrad/ results in the slower $O(1/\epsilon^2)$ rate~\citep{levy2018online, vaswani2020adaptive} compared to the $O(n + \frac{1}{\epsilon})$ rate of VR methods, we do not include it in the main paper. We demonstrate the poor performance of \adagrad/ on two example datasets in~\cref{fig:comparison-adagrad} in~\cref{app:additional-experiments}.

We plot the gradient norm of the training objective (for the best step-size) against the number of gradient evaluations normalized by the number of examples. We show the results for the logistic loss (\cref{fig:main-logistic-3datasets-bs=64}), Huber loss (\cref{fig:main-huber-3datasets-bs=64}), and squared loss (\cref{fig:4datasets-squared-epoch-bs=64}). Our results show that (i) both variants of \adasvrg/ (\emph{without any step-size tuning}) are competitive with the other best-tuned VR methods, often out-performing them or matching their performance; (ii) \svrg/-BB often has an oscillatory behavior, even for the best step-size; and (iii) the performance of \adasvrg/ with adaptive termination (that has superior theoretical complexity) is competitive with that of the practically useful $O(n)$ fixed inner-loop setting. 

In order to evaluate the effect of the step-size on a method's performance, we plot the gradient norm after $50$ outer-loops vs step-size for each of the competing methods. For the \adasvrg/ variants, we set the step-size according to the heuristic described earlier. For the logistic loss (\cref{fig:main-stepsize-logistic-3datasets-bs=64}), Huber loss (\cref{fig:main-stepsize-huber-3datasets-bs=64}) and squared loss (\cref{fig:4datasets-squared-stepsize-bs=64}), we observe that (i) the performance of typical VR methods heavily depends on the choice of the step-size; (ii) the step-size corresponding to the minimum loss is different for each method, loss and dataset; and (iii) \adasvrg/ with the step-size heuristic results in competitive performance. Additional experiments in~\cref{app:additional-experiments} confirm that the good performance of \adasvrg/ is consistent across losses, batch-sizes and datasets. 
\vspace{-1ex}
\section{Discussion}
\label{sec:conclusion}
\vspace{-1ex}
Although there have been numerous papers on VR methods in the past ten years, all of the provably convergent methods require knowledge of problem-dependent constants such as $L$. On the other hand, there has been substantial progress in designing adaptive gradient methods that have effectively replaced SGD for training ML models. Unfortunately, this progress has not been leveraged for developing better VR methods. Our work is the first to marry these lines of literature by designing \adasvrg/, that achieves a gradient complexity comparable to typical VR methods, but without needing to know the objective's smoothness constant. Our results illustrate that it is possible to design principled techniques that can ``painlessly'' reduce the variance, achieving good theoretical and practical performance. We believe that our paper will help open up an exciting research direction. In the future, we aim to extend our theory to the strongly-convex setting. 





 
\section{Acknowledgments}
We would like to thank Raghu Bollapragada for helpful discussions. This research was partially supported by the Canada CIFAR AI Chair Program, a Google Focused Research award and an IVADO postdoctoral scholarship. Simon Lacoste-Julien is a CIFAR Associate Fellow in the Learning in Machines \& Brains program.

\bibliographystyle{apalike}
\bibliography{ref}

\begin{thebibliography}{}

\bibitem[Ahn et~al., 2020]{ahn2020sgd}
Ahn, K., Yun, C., and Sra, S. (2020).
\newblock {SGD} with shuffling: optimal rates without component convexity and
  large epoch requirements.
\newblock In {\em Neural Information Processing Systems 2020, NeurIPS 2020}.

\bibitem[Allen-Zhu, 2017]{allen2017katyusha}
Allen-Zhu, Z. (2017).
\newblock Katyusha: The first direct acceleration of stochastic gradient
  methods.
\newblock In {\em Proceedings of the 49th Annual {ACM} {SIGACT} Symposium on
  Theory of Computing, {STOC}}.

\bibitem[Armijo, 1966]{armijo1966minimization}
Armijo, L. (1966).
\newblock Minimization of functions having lipschitz continuous first partial
  derivatives.
\newblock {\em Pacific Journal of mathematics}, 16(1):1--3.

\bibitem[Babanezhad~Harikandeh et~al., 2015]{babanezhad2015stopwasting}
Babanezhad~Harikandeh, R., Ahmed, M.~O., Virani, A., Schmidt, M.,
  Kone{\v{c}}n{\`y}, J., and Sallinen, S. (2015).
\newblock Stop wasting my gradients: Practical {SVRG}.
\newblock {\em Advances in Neural Information Processing Systems},
  28:2251--2259.

\bibitem[Barzilai and Borwein, 1988]{barzilai1988two}
Barzilai, J. and Borwein, J.~M. (1988).
\newblock Two-point step size gradient methods.
\newblock {\em IMA journal of numerical analysis}, 8(1):141--148.

\bibitem[Belkin et~al., 2019]{belkin2019does}
Belkin, M., Rakhlin, A., and Tsybakov, A.~B. (2019).
\newblock Does data interpolation contradict statistical optimality?
\newblock In {\em The 22nd International Conference on Artificial Intelligence
  and Statistics}, pages 1611--1619. PMLR.

\bibitem[Bollapragada et~al., 2019]{bollapragada2019exact}
Bollapragada, R., Byrd, R.~H., and Nocedal, J. (2019).
\newblock Exact and inexact subsampled newton methods for optimization.
\newblock {\em IMA Journal of Numerical Analysis}, 39(2):545--578.

\bibitem[Boucheron et~al., 2005]{boucheron2005theory}
Boucheron, S., Bousquet, O., and Lugosi, G. (2005).
\newblock Theory of classification: A survey of some recent advances.
\newblock {\em ESAIM: probability and statistics}, 9:323--375.

\bibitem[Chang and Lin, 2011]{libsvm}
Chang, C.-C. and Lin, C.-J. (2011).
\newblock {LIBSVM}: A library for support vector machines.
\newblock {\em ACM Transactions on Intelligent Systems and Technology},
  2(3):1--27.
\newblock Software available at \url{http://www.csie.ntu.edu.tw/~cjlin/libsvm}.

\bibitem[Cutkosky and Boahen, 2017]{cutkosky2017online}
Cutkosky, A. and Boahen, K. (2017).
\newblock Online convex optimization with unconstrained domains and losses.
\newblock {\em arXiv preprint arXiv:1703.02622}.

\bibitem[Cutkosky and Orabona, 2019]{cutkosky2019momentum}
Cutkosky, A. and Orabona, F. (2019).
\newblock Momentum-based variance reduction in non-convex {SGD}.
\newblock {\em arXiv preprint arXiv:1905.10018}.

\bibitem[Defazio et~al., 2014]{defazio2014saga}
Defazio, A., Bach, F., and Lacoste-Julien, S. (2014).
\newblock {SAGA}: A fast incremental gradient method with support for
  non-strongly convex composite objectives.
\newblock In {\em Advances in Neural Information Processing Systems,
  {N}eur{IPS}}.

\bibitem[Defazio and Bottou, 2019]{defazio2018ineffectiveness}
Defazio, A. and Bottou, L. (2019).
\newblock On the ineffectiveness of variance reduced optimization for deep
  learning.
\newblock In {\em Advances in Neural Information Processing Systems,
  {NeurIPS}}.

\bibitem[Duchi et~al., 2011]{duchi2011adaptive}
Duchi, J.~C., Hazan, E., and Singer, Y. (2011).
\newblock Adaptive subgradient methods for online learning and stochastic
  optimization.
\newblock {\em The Journal of Machine Learning Research}, 12:2121--2159.

\bibitem[Gower et~al., 2020]{gower2020variance}
Gower, R.~M., Schmidt, M., Bach, F., and Richtarik, P. (2020).
\newblock Variance-reduced methods for machine learning.
\newblock {\em Proceedings of the IEEE}, 108(11):1968--1983.

\bibitem[Hofmann et~al., 2015]{hofmann2015variance}
Hofmann, T., Lucchi, A., Lacoste-Julien, S., and McWilliams, B. (2015).
\newblock Variance reduced stochastic gradient descent with neighbors.
\newblock {\em Advances in Neural Information Processing Systems},
  28:2305--2313.

\bibitem[Johnson and Zhang, 2013]{johnson2013accelerating}
Johnson, R. and Zhang, T. (2013).
\newblock Accelerating stochastic gradient descent using predictive variance
  reduction.
\newblock In {\em Advances in Neural Information Processing Systems,
  {N}eur{IPS}}.

\bibitem[Kone{\v{c}}n{\`y} and Richt{\'a}rik, 2013]{konevcny2013semi}
Kone{\v{c}}n{\`y}, J. and Richt{\'a}rik, P. (2013).
\newblock Semi-stochastic gradient descent methods.
\newblock {\em arXiv preprint arXiv:1312.1666}.

\bibitem[Kovalev et~al., 2020]{kovalev2020don}
Kovalev, D., Horv{\'a}th, S., and Richt{\'a}rik, P. (2020).
\newblock Don’t jump through hoops and remove those loops: {SVRG} and
  {K}atyusha are better without the outer loop.
\newblock In {\em Algorithmic Learning Theory}, pages 451--467. PMLR.

\bibitem[Lan et~al., 2019]{lan2019unified}
Lan, G., Li, Z., and Zhou, Y. (2019).
\newblock A unified variance-reduced accelerated gradient method for convex
  optimization.
\newblock In {\em Advances in Neural Information Processing Systems}, pages
  10462--10472.

\bibitem[Lang et~al., 2019]{lang2019using}
Lang, H., Xiao, L., and Zhang, P. (2019).
\newblock Using statistics to automate stochastic optimization.
\newblock In {\em Advances in Neural Information Processing Systems}, pages
  9540--9550.

\bibitem[Levy et~al., 2018]{levy2018online}
Levy, K.~Y., Yurtsever, A., and Cevher, V. (2018).
\newblock Online adaptive methods, universality and acceleration.
\newblock In {\em Advances in Neural Information Processing Systems,
  {NeurIPS}}.

\bibitem[Li et~al., 2020]{li2020almost}
Li, B., Wang, L., and Giannakis, G.~B. (2020).
\newblock Almost tune-free variance reduction.
\newblock In {\em International Conference on Machine Learning}, pages
  5969--5978. PMLR.

\bibitem[Liang et~al., 2020]{liang2020just}
Liang, T., Rakhlin, A., et~al. (2020).
\newblock Just interpolate: Kernel “ridgeless” regression can generalize.
\newblock {\em Annals of Statistics}, 48(3):1329--1347.

\bibitem[Liu et~al., 2020]{liu2020adam}
Liu, M., Zhang, W., Orabona, F., and Yang, T. (2020).
\newblock Adam+: A stochastic method with adaptive variance reduction.
\newblock {\em arXiv preprint arXiv:2011.11985}.

\bibitem[Loizou et~al., 2020]{loizou2020stochastic}
Loizou, N., Vaswani, S., Laradji, I., and Lacoste-Julien, S. (2020).
\newblock Stochastic {P}olyak step-size for {SGD}: An adaptive learning rate
  for fast convergence.
\newblock {\em arXiv preprint:2002.10542}.

\bibitem[Ma et~al., 2018]{ma2018power}
Ma, S., Bassily, R., and Belkin, M. (2018).
\newblock The power of interpolation: Understanding the effectiveness of {SGD}
  in modern over-parametrized learning.
\newblock In {\em Proceedings of the 35th International Conference on Machine
  Learning, {ICML}}.

\bibitem[Mahdavi and Jin, 2013]{mahdavi2013mixedgrad}
Mahdavi, M. and Jin, R. (2013).
\newblock Mixed{G}rad: An {O}(1/{T}) convergence rate algorithm for stochastic
  smooth optimization.
\newblock {\em arXiv preprint arXiv:1307.7192}.

\bibitem[Mairal, 2013]{mairal2013optimization}
Mairal, J. (2013).
\newblock Optimization with first-order surrogate functions.
\newblock In {\em International Conference on Machine Learning}, pages
  783--791.

\bibitem[Meng et~al., 2020]{meng2019fast}
Meng, S.~Y., Vaswani, S., Laradji, I., Schmidt, M., and Lacoste{-}Julien, S.
  (2020).
\newblock Fast and furious convergence: Stochastic second order methods under
  interpolation.
\newblock In {\em The 23nd International Conference on Artificial Intelligence
  and Statistics, {AISTATS}}.

\bibitem[Nesterov, 2004]{nesterov2013introductory}
Nesterov, Y. (2004).
\newblock {\em Introductory lectures on convex optimization: A basic course}.
\newblock Springer Science \& Business Media.

\bibitem[Nguyen et~al., 2017]{nguyen2017sarah}
Nguyen, L.~M., Liu, J., Scheinberg, K., and Tak{\'a}{\v{c}}, M. (2017).
\newblock {SARAH}: a novel method for machine learning problems using
  stochastic recursive gradient.
\newblock In {\em Proceedings of the 34th International Conference on Machine
  Learning-Volume 70}, pages 2613--2621.

\bibitem[Pesme et~al., 2020]{pesme2020convergence}
Pesme, S., Dieuleveut, A., and Flammarion, N. (2020).
\newblock On convergence-diagnostic based step sizes for stochastic gradient
  descent.
\newblock {\em arXiv preprint arXiv:2007.00534}.

\bibitem[Pflug, 1983]{pflug1983determination}
Pflug, G.~C. (1983).
\newblock On the determination of the step size in stochastic quasigradient
  methods.

\bibitem[Qian and Qian, 2019]{qian2019implicit}
Qian, Q. and Qian, X. (2019).
\newblock The implicit bias of adagrad on separable data.
\newblock {\em arXiv preprint arXiv:1906.03559}.

\bibitem[Reddi et~al., 2016]{reddi2016stochastic}
Reddi, S.~J., Hefny, A., Sra, S., Poczos, B., and Smola, A. (2016).
\newblock Stochastic variance reduction for nonconvex optimization.
\newblock In {\em International conference on machine learning}, pages
  314--323.

\bibitem[Reddi et~al., 2018]{reddi2018convergence}
Reddi, S.~J., Kale, S., and Kumar, S. (2018).
\newblock On the convergence of {A}dam and {B}eyond.
\newblock In {\em International Conference on Learning Representations}.

\bibitem[Schmidt et~al., 2015]{schmidt2015non}
Schmidt, M., Babanezhad, R., Ahmed, M., Defazio, A., Clifton, A., and Sarkar,
  A. (2015).
\newblock Non-uniform stochastic average gradient method for training
  conditional random fields.
\newblock In {\em Proceedings of the Eighteenth International Conference on
  Artificial Intelligence and Statistics, {AISTATS}}.

\bibitem[Schmidt and Le~Roux, 2013]{schmidt2013fast}
Schmidt, M. and Le~Roux, N. (2013).
\newblock Fast convergence of stochastic gradient descent under a strong growth
  condition.
\newblock {\em arXiv preprint:1308.6370}.

\bibitem[Schmidt et~al., 2017]{schmidt2017minimizing}
Schmidt, M., Le~Roux, N., and Bach, F. (2017).
\newblock Minimizing finite sums with the stochastic average gradient.
\newblock {\em Mathematical Programming}, 162(1-2):83--112.

\bibitem[Sebbouh et~al., 2019]{sebbouh2019towards}
Sebbouh, O., Gazagnadou, N., Jelassi, S., Bach, F., and Gower, R. (2019).
\newblock Towards closing the gap between the theory and practice of {SVRG}.
\newblock In {\em Advances in Neural Information Processing Systems}, pages
  648--658.

\bibitem[Shalev-Shwartz and Zhang, 2013]{shalev2013stochastic}
Shalev-Shwartz, S. and Zhang, T. (2013).
\newblock Stochastic dual coordinate ascent methods for regularized loss
  minimization.
\newblock {\em Journal of Machine Learning Research}, 14(Feb):567--599.

\bibitem[Song et~al., 2020]{song2020variance}
Song, C., Jiang, Y., and Ma, Y. (2020).
\newblock Variance reduction via accelerated dual averaging for finite-sum
  optimization.
\newblock {\em Advances in Neural Information Processing Systems}, 33.

\bibitem[Sridharan et~al., 2008]{sridharan2008fast}
Sridharan, K., Shalev-Shwartz, S., and Srebro, N. (2008).
\newblock Fast rates for regularized objectives.
\newblock {\em Advances in neural information processing systems},
  21:1545--1552.

\bibitem[Tan et~al., 2016]{tan2016barzilai}
Tan, C., Ma, S., Dai, Y.-H., and Qian, Y. (2016).
\newblock Barzilai-{B}orwein step size for stochastic gradient descent.
\newblock {\em arXiv preprint arXiv:1605.04131}.

\bibitem[Vaswani et~al., 2019a]{vaswani2019fast}
Vaswani, S., Bach, F., and Schmidt, M. (2019a).
\newblock Fast and faster convergence of {SGD} for over-parameterized models
  and an accelerated perceptron.
\newblock In {\em The 22nd International Conference on Artificial Intelligence
  and Statistics, {AISTATS}}.

\bibitem[Vaswani et~al., 2020]{vaswani2020adaptive}
Vaswani, S., Kunstner, F., Laradji, I., Meng, S.~Y., Schmidt, M., and
  Lacoste-Julien, S. (2020).
\newblock Adaptive gradient methods converge faster with over-parameterization
  (and you can do a line-search).
\newblock {\em arXiv preprint arXiv:2006.06835}.

\bibitem[Vaswani et~al., 2019b]{vaswani2019painless}
Vaswani, S., Mishkin, A., Laradji, I., Schmidt, M., Gidel, G., and
  Lacoste-Julien, S. (2019b).
\newblock Painless stochastic gradient: Interpolation, line-search, and
  convergence rates.
\newblock In {\em Advances in Neural Information Processing Systems,
  {N}eur{IPS}}.

\bibitem[Ward et~al., 2019]{ward2018adagrad}
Ward, R., Wu, X., and Bottou, L. (2019).
\newblock {A}da{G}rad stepsizes: Sharp convergence over nonconvex landscapes,
  from any initialization.
\newblock In {\em Proceedings of the 36th International Conference on Machine
  Learning, {ICML}}.

\bibitem[Yaida, 2018]{yaida2018fluctuation}
Yaida, S. (2018).
\newblock Fluctuation-dissipation relations for stochastic gradient descent.
\newblock {\em arXiv preprint arXiv:1810.00004}.

\bibitem[Zhang et~al., 2017]{zhang2016understanding}
Zhang, C., Bengio, S., Hardt, M., Recht, B., and Vinyals, O. (2017).
\newblock Understanding deep learning requires rethinking generalization.
\newblock In {\em 5th International Conference on Learning Representations,
  {ICLR}}.

\end{thebibliography}

\clearpage
\onecolumn
\appendix
\newcommand{\appendixTitle}{%
\vbox{
    \centering
	\hrule height 4pt
	\vskip 0.2in
	{\LARGE \bf Supplementary material}
	\vskip 0.2in
	\hrule height 1pt 
}}

\newcommand{\makeappendixtable}[1]{%
~\\[.75em]\hskip 2em \begin{tabular}{@{}p{0.44\textwidth}p{0.2\textwidth}p{0.14\textwidth}@{}}
\toprule
#1\\
\bottomrule
\end{tabular}\\[.0em]
}
\newcommand{\makeappendixheader}{%
\textbf{Step-size} & \textbf{Rate} & \textbf{Reference} \\
\midrule
}
\appendixTitle

\section*{Organization of the Appendix}
\begin{itemize}
    \item[\ref{app:definitions}] \nameref{app:definitions}
    
    \item[\ref{app:interpolation}] \nameref{app:interpolation}
  
    \item[\ref{app:lemma:adasvrg-outerloop}] \nameref{app:lemma:adasvrg-outerloop}
    
    \item[\ref{app:prop:main}] \nameref{app:prop:main}
    
    \item[\ref{app:thm:adasvrg-main}] \nameref{app:thm:adasvrg-main}
    
    \item[\ref{app:thm:adasvrg-doubling}] \nameref{app:thm:adasvrg-doubling}
    
    \item[\ref{app:thm:adasvrg-at}] \nameref{app:thm:adasvrg-at}
    
    \item[\ref{app:thm:phase_trns}] \nameref{app:thm:phase_trns}
    
    \item[\ref{app:helper-lemmas}] \nameref{app:helper-lemmas}    
    
    \item[\ref{app:line-search-counter-example}] \nameref{app:line-search-counter-example}    
    
    \item[\ref{app:additional-experiments}] 
    \nameref{app:additional-experiments}    
    
\end{itemize}

\section{Definitions}
\label{app:definitions}
Our main assumptions are that each individual function $f_i$ is differentiable and 
$L_i$-smooth, meaning that for all $v$ and $w$, 
\aligns{
    f_i(v) & \leq f_i(w) + \inner{\nabla f_i(w)}{v - w} + \frac{L_i}{2} \normsq{v - w},
    \tag{Individual Smoothness}
    \label{eq:individual-smoothness}
}
which also implies that $f$ $L_{\max}$-smooth, where $L_{\max}$ 
is the maximum smoothness constant of the individual functions.
We also assume that $f$ is convex, meaning that for all $v$ and $w$,
\aligns{
    f(v) &\geq f(w) - \lin{\nabla f(w), w-v}.
    \tag{Convexity}
    \label{eq:individual-convexity}
    \\
}
\newpage

\section{Heuristic for adaptivity to over-parameterization}
\label{app:interpolation}
\begin{algorithm}[H]
\caption{Hybrid \adagrad/-\adasvrg/}
\label{alg:hybrid}
\textbf{Input}: $x_1$ (initial point), $T$ (iteration budget), $M$ (maximum inner loops), $\theta$ (adaptive termination parameter) \\
\textbf{Phase 1: \adagrad/} \\
    $G_0 \leftarrow \mathbf{0}$ \\
    \For{$t \leftarrow 1$ \KwTo $T$}{
    Sample $i_t \sim \text{Uniform}\{1, 2, \ldots n\}$ \\
    $\eta_t \leftarrow$ Compute step-size \\
    $G_t = G_{t-1} + \normsq{g_t}$ \\
    $A_t = G_t^{1/2}$\\
    \If{$t$ mod $2$ $==0$ and $t \geq 2n$}
    {
        $R = \frac{\|G_{t}\|^2_* - \|G_{t/2}\|^2_*}{\|G_{t/2}\|^2_*}$\\
        \If{$R \geq \theta$}{
            Terminate and switch to AdaSVRG
        }
    }
    $x_{t+1} = \proj{X, A_t}{x_t - \eta_{k} A_t^{-1}g_t}$
    }
    \textbf{Phase 2: \adasvrg/} \\
    $w_{K}$ = \cref{alg:adasvrg-at-test}($x_t$, $\floor{\frac{T - t}{n}}$, $\theta$, $M$) \\
    Return $w_{K}$
\end{algorithm}

In this section, we reason that the poor empirical performance of \svrg/ when training over-parameterized models~\citep{defazio2018ineffectiveness} can be partially explained by the interpolation property~\citep{schmidt2013fast, ma2018power, vaswani2019fast} satisfied by these models~\citep{zhang2016understanding}. In particular, we focus on smooth convex losses, but assume that the model is capable of completely fitting the training data, and that $w^*$ lies in the interior of $X$. For example, these properties are simultaneously satisfied when minimizing the squared hinge-loss for linear classification on separable data or unregularized kernel regression~\citep{belkin2019does,liang2020just} with $\|w^*\| \leq 1$. 

Formally, the interpolation condition means that the gradient of \emph{each} $f_i$ in the finite-sum converges to zero at an optimum. Additionally, we assume that each function $f_i$ has finite minimum $\fjopt$. If the overall objective $f$ is minimized at $\xopt$, $\grad{\xopt} = 0$, then for all $\fj$ we have $\gradi{\xopt} = 0$. Since the interpolation property is rarely exactly satisfied in practice, we allow for a weaker version that uses $\zeta^2 \coloneqq \E_i [f^* - \fjopt ] \in [0,\infty)$~\citep{loizou2020stochastic, vaswani2020adaptive} to measure the extent of the violation of interpolation. If $\zeta^2 = 0$, interpolation is exactly satisfied. 

When $\zeta^2 = 0$, both constant step-size SGD and \adagrad/ have a gradient complexity of $O(1/\epsilon)$ in the smooth convex setting~\citep{schmidt2013fast, vaswani2019fast, vaswani2020adaptive}. In contrast, typical VR methods have an $\tilde{O}(n + \frac{1}{\epsilon})$ complexity. For example, both \svrg/ and \adasvrg/ require computing the full gradient in every outer-loop, and will thus unavoidably suffer an $\Omega(n)$ cost. For large $n$, typical VR methods will thus be necessarily slower than SGD when training models that can exactly interpolate the data. This provides a partial explanation for the ineffectiveness of VR methods when training over-parameterized models. When $\zeta^2 > 0$, \adagrad/ has an $O(1/\epsilon + \zeta/\epsilon^2)$ rate~\citep{vaswani2020adaptive}. Here $\zeta$, the violation of interpolation plays the role of noise and slows down the convergence to an $O(1/\epsilon^2)$ rate. On the other hand, \adasvrg/ results in an $\tilde{O}(n + 1/\epsilon)$ rate, regardless of $\zeta$. 

Following the reasoning in~\cref{sec:adap-termination}, if an algorithm can detect the slower convergence of \adagrad/ and switch from \adagrad/ to \adasvrg/, it can attain a faster convergence rate. It is straightforward to show that \adagrad/ has a a similar phase transition as~\cref{thm:phase_trns} when interpolation is only approximately satisfied. This enables the use of the test in~\cref{sec:adap-termination} to terminate \adagrad/ and switch to \adasvrg/, resulting in the hybrid algorithm described in~\cref{alg:hybrid}. If the diagnostic test can detect the phase transition accurately,~\cref{alg:hybrid} will attain an $O(1/\epsilon)$ convergence when interpolation is exactly satisfied (no switching in this case). When interpolation is only approximately satisfied, it will result in an $O(1/\epsilon)$ convergence for $\epsilon \geq \zeta$ (corresponding to the \adagrad/ rate in the deterministic phase) and will attain an $O(1/\zeta^2 + ((n + 1/\epsilon) \log(\zeta/\epsilon))$ convergence thereafter (corresponding to the \adasvrg/ rate). This implies that~\cref{alg:hybrid} can indeed obtain the best of both worlds between \adagrad/ and \adasvrg/.

\begin{figure*}[!h]
\centering     
\includegraphics[scale = 0.277]{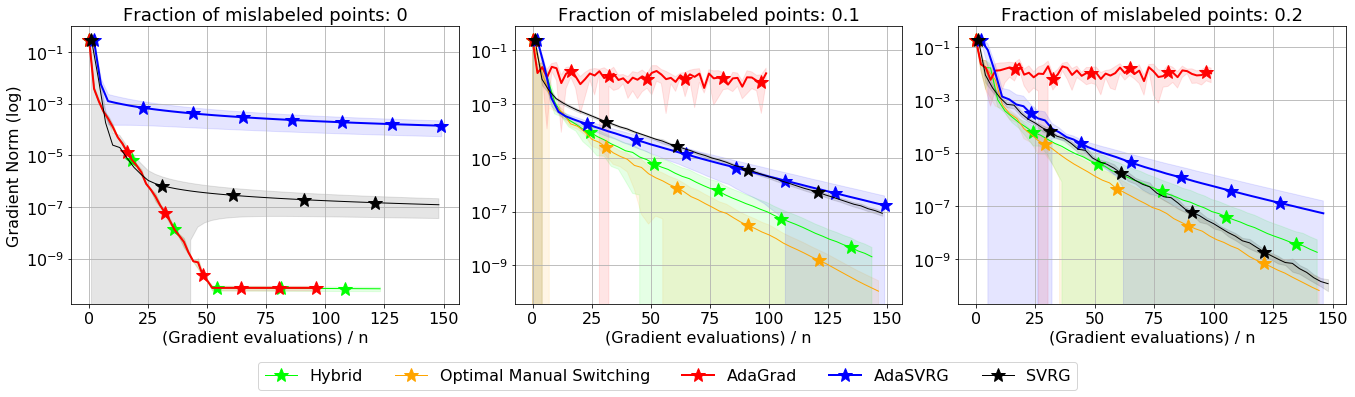}
\caption{Comparison of \adagrad/, \adasvrg/ and~\cref{alg:hybrid} (denoted "Hybrid" in the plots) with logistic loss and batch-size $64$ on datasets with different fraction of mislabeled data-points. Interpolation is exactly satisfied for the left-most plot.}
\label{fig:main-interpolation-logistic-bs=64}
\end{figure*}

\textbf{Evaluating~\cref{alg:hybrid}}: We use synthetic experiments to demonstrate the effect of interpolation on the convergence of stochastic and VR methods. Following the protocol in~\citep{meng2019fast}, we generate a linearly separable dataset with $n=10^4$ data points of dimension $d=200$ and train a linear model with a convex loss. This setup ensures that interpolation is satisfied, but allows to eliminate other confounding factors such as non-convexity and other implementation details. In order to smoothly violate interpolation, we show results with a mislabel fraction of points in the grid $[0, 0.1, 0.2]$. 

We use \adagrad/ as a representative (fully) stochastic method, and to eliminate possible confounding because of its step-size, we set it using the stochastic line-search procedure~\citep{vaswani2020adaptive}. We compare the performance of \adagrad/, \svrg/, \adasvrg/ and the hybrid \adagrad/-\adasvrg/ (\cref{alg:hybrid}) each with a budget of $50$ epochs (passes over the data). For \svrg/, as before, we choose the best step-size via a grid-search. For \adasvrg/, we use the fixed-size inner-loop variant and the step-size heuristic described earlier. In order to evaluate the quality of the ``switching'' metric in~\cref{alg:hybrid}, we compare against a hybrid method referred to as ``Optimal Manual Switching'' in the plots. This method runs a grid-search over switching points - after epoch $\{1, 2, \ldots, 50\}$ and chooses the point that results in the minimum loss after $50$ epochs.

In~\cref{fig:main-interpolation-logistic-bs=64}, we plot the results for the logistic loss using a batch-size of $64$ (refer to~\cref{app:additional-experiments} for other losses and batch-sizes). We observe that (i) when interpolation is exactly satisfied (no mislabeling), \adagrad/ results in superior performance over \svrg/ and \adasvrg/, confirming the theory in~\cref{app:interpolation}. In this case, both the optimal manual switching and~\cref{alg:hybrid} do not switch; (ii) when interpolation is not exactly satisfied (with $10\%, 20\%$ mislabeling), the \adagrad/ progress slows down to a stall in a neighbourhood of the solution, whereas both \svrg/ and \adasvrg/ converge to the solution; (iii) in both cases,~\cref{alg:hybrid} detects the slowdown in \adagrad/ and switches to \adasvrg/, resulting in competitive performance with the optimal manual switching. For all three datasets,~\cref{alg:hybrid} matches or out-performs the better of \adagrad/ and \adasvrg/, showing that it can achieve the best-of-both-worlds. 

In~\cref{app:additional-experiments}, we evaluate the performance of these methods for binary classification with kernel mappings on the \textit{mushrooms} and \textit{ijcnn} datasets. This experimental setup can also simultaneously ensure convexity and interpolation~\citep{vaswani2019painless}, and we observe the same favourable behaviour of the proposed hybrid algorithm.








\newpage
\section{Algorithm in general case}
We restate~\cref{alg:adasvrg} to handle the full matrix and diagonal variants. The only difference is in the initialization and update of $G_t$.
\label{app:algorithm-general-case}
		\begin{algorithm}[h]
\caption{\adasvrg/ with fixed-sized inner-loop}
\label{alg:adasvrg-general}
\textbf{Input}: $w_0$ (initial point), $K$ (outer-loops), $m$ (inner-loops) \\
\For{$k \leftarrow 0$  \KwTo  $K-1$}{
    Compute full gradient $\nabla f(w_k)$ \\
    $\eta_k \leftarrow$ Compute step-size \\
    Initialize: $x_1 = w_k$ and $G_0 = \begin{cases}
    0 & \text{ (scalar variant)}\\
    \delta I & \text{ (full matrix and diagonal variants)}
    \end{cases}$. \\
    \For{$t \leftarrow 1$   \KwTo  $m$}{
    Sample $i_t \sim \text{Uniform}\{1, 2, \ldots n\}$ \\
    $g_t = \nabla \fit(x_t) - \nabla \fit(w_k) + \nabla f(w_k)$\\
    $G_t = G_{t-1} + \begin{cases}
    \normsq{g_t} &\text{ (scalar variant)}\\
    \diag\left(g_t g_t^\top\right) &\text{ (diagonal variant)}\\
    g_t g_t^\top &\text{ (full matrix variant)}
    \end{cases}$ \\
    $A_t = G_t^{1/2}$\\
    $x_{t+1} = \proj{X, A_t}{x_t - \eta_{k} A_t^{-1}g_t}$
    }
    $w_{k+1} = \frac{1}{m} \sum_{t=1}^{m} x_t$
}
Return $\bar{w}_K = \frac{1}{K} \sum_{k = 1}^{K} w_{k}$
\end{algorithm}

\section{Proof of Lemma~\ref{lemma:adasvrg-outerloop}}
\label{app:lemma:adasvrg-outerloop} 
We restate~\cref{lemma:adasvrg-outerloop} to handle the three variants of \adasvrg/.

\begin{thmbox}\begin{restatable}[\adasvrg/ with single outer-loop]{lemma}{restateAdaSVRGouterloop-general}
Assuming (i) convexity of $f$, (ii) $L_{\max}$-smoothness of $f_i$ and (iii) bounded feasible set with diameter $D$. For the scalar variant, defining $\rho := \big( \frac{D^2}{\etak} + 2\etak \big) \sqrt{L_{\text{max}}}$, for any outer loop $k$ of \adasvrg/, with (a) inner-loop length $m_k$ and (b) step-size $\etak$, 
\begin{align*}
\E[f(w_{k+1}) - f^*] \leq \frac{\rho^2}{m_k} + \frac{\rho\sqrt{\E[f(w_k) - f^*]}}{\sqrt{m_k}}
\end{align*}
For the full matrix and diagonal variants, setting $\rho' := \left( \frac{D^2}{\eta_k} + 2\eta_k \right) \sqrt{dL_{\text{max}}}$,
\begin{align*}
    \E [ f(w_{k+1}) - f^*] \leq \frac{(\rho')^2 + \sqrt{\nicefrac{d\delta}{4L_{\text{max}}}}}{m_k} +  \frac{\rho' \sqrt{\E [f(w_k) - f^*]}}{\sqrt{m_k}}
\end{align*}
\label{lemma:adasvrg-outerloop-general}
\end{restatable}
\end{thmbox}
\begin{proof}
For any of the three variants, we have, for any outer loop iteration $k$ and any inner loop iteration $t$,
\begin{align}
    \norm{x_{t+1} - w^*}^2_{A_t} &= \norm{P_{X, A_t}(x_t - \eta_k A_t^{-1} g_k) - P_{X, A_t}(\xopt)}_{A_t}^2\\
    &\leq \norm{x_t - \eta_k A^{-1}_t g_k - \xopt}^2_{A_t} \\
    &= \norm{x_t - w^*}^2_{A_t} - 2\eta_k \inner{g_t}{x_t - w^*} + \eta_k^2 \norm{g_t}^2_{A_t^{-1}}
\end{align}
where the inequality follows from \citet[Lemma 4]{reddi2018convergence}
Dividing by $\eta_k$, rearranging and summing over all inner loop iterations at stage $k$ gives
\begin{align}
    2 \sum_{t=1}^{m_k} \inner{g_t}{x_t - w^*} &\leq \sum_{t=1}^{m_k} \norm{x_t - w^*}^2_{\frac{A_t}{\eta_k} - \frac{A_{t-1}}{\eta_k}} + \sum_{t=1}^{m_k} \eta_k \norm{g_t}^2_{A_t^{-1}}\\
    &\leq \frac{D^2}{\eta_k}\mathrm{Tr}(A_{m_k}) + 2\eta_k \mathrm{Tr}(A_{m_k}) & \tag{\cref{lem:telescoping} and~\cref{lem:trace-ineq} } \\
    &= \label{eq:lem1-ineq1}\left(\frac{D^2}{\eta_k} + 2\eta_k \right) \mathrm{Tr}(A_{m_k})
\end{align}

By~\cref{lem:trace-ineq}, we have that $\Tr{A_{m_k}} \leq \sqrt{\sum_{t=1}^{m_k} \norm{g_t}^2}$ in the scalar case, and $\Tr{A_{m_k}} \leq \sqrt{d}\sqrt{\sum_{t=1}^{m_k} \norm{g_t}^2 + d\delta}$ in the full matrix and diagonal variants.
Therefore we set 
\begin{align*}
    a' = \begin{cases}
    \frac{1}{2}\big( \frac{D^2}{\eta_k} + 2\eta_k\big) & \text{(scalar variant)}\\
    \frac{1}{2}\big( \frac{D^2}{\eta_k} + 2\eta_k\big)\sqrt{d} &\text{(full matrix and diagonal variants)}
    \end{cases}
\end{align*}
and
\begin{align*}
    b = \begin{cases}
    0 & \text{(scalar variant)}\\
    d \delta  &\text{(full matrix and diagonal variants)}
    \end{cases}
\end{align*}
Going back to the above inequality and taking expectation we get

{\begin{align}
\sum_{t=1}^{m_k} \inner{\nabla f(x_t)}{x_t - w^*} &\leq a' \, \E \left[\sqrt{\sum_{t=1}^{m_k} \norm{g_t}^2 + b} \right]
\end{align}}

Using convexity of $f$ yields

{
\begin{align}
    \sum_{t=1}^{m_k} \E[f(x_t) - f^*] &\leq a' \E\bigg[ \sqrt{\sum_{t=1}^{m_k} \norm{g_t}^2+ b}\bigg]\\
    &\leq a' \sqrt{\E\big[ \sum_{t=1}^{m_k} \norm{g_t}^2+ b\big]}\\
    &= a' \sqrt{\sum_{t=1}^{m_k} \E[\norm{g_t}^2] + b}
\end{align}}

where the second inequality comes from Jensen's inequality applied to the (concave) square root function. Now, from \cite{kovalev2020don},
\begin{align}
    \E[\norm{g_t}^2] \leq 4L_{\text{max}} \E[f(w_k) - f^*] + 4L_{\text{max}}\E[f(x_t) - f^*]
\end{align}

Going back to the previous equation, squaring and setting $\tau = a' \sqrt{4 L_{\text{max}}}$ we get
\begin{align}
    \left(\sum_{t=1}^{m_k} \E[f(x_t) - f^*]\right)^2 \leq \tau^2 \left( \sum_{t=1}^{m_k} \E[f(x_t) - f^*] + m_k\E[f(w_k) - f^*] + \frac{b}{4 L_{\text{max}}} \right)
\end{align}

Using~\cref{lem:quad_ineq}, 
\begin{align}
    \sum_{t=1}^{m_k} \E[f(x_t) - f^*] \leq \tau^2 + \tau \sqrt{m_k\E[f(w_k) - f^*] + \frac{b}{4L_{\text{max}}}} \leq \tau^2 + \tau \sqrt{\frac{b}{4L_{\text{max}}} } + \tau \sqrt{m_k \E[f(w_k) - f^*]}
\end{align}

Finally, using Jensen's inequality we get
\begin{align}
    \E
    \left[f(w_{k+1}) - f^*\right] &= \E\left[f\left(\frac{1}{m_k}\sum_{t=1}^{m_k} x_t\right) - f^*\right]\\
    &\leq \frac{\tau^2 + \tau \sqrt{\frac{b}{4L_{\text{max}}} }}{m_k} + \frac{\tau\sqrt{\E[f(w_k) - f^*]}}{\sqrt{m_k}}
\end{align}
which concludes the proof by noticing that by definition $\tau = \rho$ in the scalar case and $\tau = \rho'$ in the full matrix and diagonal cases.

\end{proof}

\newpage
\section{Main Proposition}
\label{app:prop:main} 
We first state the main proposition for the three variants of \adasvrg/, which we later use for proving theorems. 
\begin{proposition}
\label{apdx:prop:adasvrg}
Assuming (i) convexity of $f$ (ii) $L_{\text{max}}$-smoothness of $f$ (iii) bounded feasible set (iv) $\etak \in [\eta_{\text{min}}, \eta_{\text{max}}]$ (v) $m_k = m$ for all $k$, then for the scalar variant,
\begin{align*}
    \E[ f(\bar{w}_K) - f^*] \leq \frac{\rho^2 ( 1 + \sqrt{5})}{m} +  \frac{\rho\sqrt{2\left(f(w_0) - f^*\right)}}{\sqrt{mK}}
\end{align*}
and for the full matrix and diagonal variants,
\begin{align*}
    \E[ f(\bar{w}_K) - f^*] \leq \frac{(\rho')^2 ( 1 + \sqrt{5}) + 2\rho'\sqrt{\nicefrac{d\delta}{L_{\text{max}}}}}{m} +  \frac{\rho'\sqrt{2\left(f(w_0) - f^*\right)}}{\sqrt{mK}}
\end{align*}
where $\bar{w}_K = \frac{1}{K}\sum_{k=1}^{K}w_k$, $\rho =  \left( \frac{D^2}{\eta_{\text{min}}} + 2\eta_{\text{max}}\right) \sqrt{L_{\text{max}}}$ and $\rho' =  \left( \frac{D^2}{\eta_{\text{min}}} + 2\eta_{\text{max}}\right) \sqrt{dL_{\text{max}}}$.
\end{proposition}

\begin{proof}
As in the previous proof we define
\begin{align*}
    b = \begin{cases}
    0 & \text{(scalar variant)}\\
    d \delta  &\text{(full matrix and diagonal variants)}
    \end{cases}
\end{align*}
and
\begin{align*}
    \tau = \begin{cases}
    \rho & \text{(scalar variant)}\\
    \rho'  &\text{(full matrix and diagonal variants)}
    \end{cases}
\end{align*}
Using the result from Lemma~\ref{lemma:adasvrg-outerloop} and letting $\Delta_k \coloneqq \E[f(w_k) - f^*]$, we have, 
\begin{align}
    \Delta_{k+1} \leq \frac{\tau^2 + \tau \sqrt{\frac{b}{4L_{\text{max}}}}}{m} + \tau\frac{\sqrt{\Delta_k}}{\sqrt{m}}
\end{align}
Squaring gives
\begin{align}
    \Delta_{k+1}^2 &\leq \left( \frac{\tau^2 +  \tau \sqrt{\frac{b}{4L_{\text{max}}}}}{m} + \frac{\tau\sqrt{\Delta_k}}{\sqrt{m}}\right)^2 \leq 2 \frac{\left(\tau^2 +  \tau \sqrt{\frac{b}{4L_{\text{max}}}}\right)^2}{m^2} + 2\frac{\tau^2}{m} \Delta_k \leq 4 \frac{\tau^2}{m^2} + 4 \frac{\tau^2 \frac{b}{4 L_{\text{max}}}}{m^2} + 2 \frac{\tau^2 }{m} \Delta_k
\end{align}
which we can rewrite as
\begin{align}
    \Delta_{k+1}^2 - 2\frac{\tau^2}{m} \Delta_{k+1} \leq 4 \frac{\tau^4}{m^2} + 2\frac{\tau^2}{m} (\Delta_k -\Delta_{k+1}) + 4 \frac{\tau^2 b}{L_{\text{max}} m^2}
\end{align}
Since $\Delta_{k+1}^2 - 2\frac{\tau^2}{m} \Delta_{k+1} = (\Delta_{k+1} - \frac{\tau^2}{m})^2 - \frac{\tau^4}{m^2}$, we get
\begin{align}
    (\Delta_{k+1} - \frac{\tau^2}{m})^2 \leq 5 \frac{\tau^4}{m^2} + 2\frac{\tau^2}{m} (\Delta_k -\Delta_{k+1}) + 4 \frac{\tau^2 b}{L_{\text{max}}m^2}
\end{align}
Summing this gives
\begin{align}
    \sum_{k=0}^{K-1} \left(\Delta_{k+1} - \frac{\tau^2}{m}\right)^2 &\leq \left(5\tau^4 + 4\frac{\tau^2  b}{L_{\text{max}}}\right) \frac{K}{m^2} + 2\frac{\tau^2}{m} \sum_{k=0}^{K-1}(\Delta_k - \Delta_{k+1})\\
    &\leq \left(5\tau^4 + 4\frac{\tau^2  b}{L_{\text{max}}}\right) \frac{K}{m^2} + 2\frac{\tau^2}{m} \Delta_0
\end{align}
Using Jensen's inequality on the (concave) square root function gives
\begin{align}
    \frac{1}{K}\sum_{k=0}^{K-1} (\Delta_{k+1} - \frac{\tau^2}{m}) \leq \frac{1}{K}\sum_{k=0}^{K-1} \sqrt{\left(\Delta_{k+1} - \frac{\tau^2}{m}\right)^2} \leq \sqrt{\frac{1}{K}\sum_{k=0}^{K-1} \left(\Delta_{k+1} - \frac{\tau^2}{m}\right)^2}
\end{align}
going back to the previous inequality this gives
\begin{align}
    \frac{1}{K}\sum_{k=0}^{K-1} (\Delta_{k+1} - \frac{\tau^2}{m}) &\leq \sqrt{\left(5\tau^4 + 4\frac{\tau^2  b}{L_{\text{max}}}\right) \frac{1}{m^2} + 2\tau^2\frac{\Delta_0}{mK}  }\\
    &\leq \frac{\tau^2\sqrt{5} + 2\tau\sqrt{\nicefrac{b}{L_{\text{max}}}}}{m}  + \frac{\tau\sqrt{2\Delta_0}}{\sqrt{mK}}
\end{align}
which we can rewrite
\begin{align}
    \frac{1}{K}\sum_{k=0}^{K-1} \Delta_{k+1} \leq \frac{\tau^2( 1 + \sqrt{5}) + 2\tau\sqrt{\nicefrac{b}{L_{\text{max}}}}}{m}  + \frac{\tau\sqrt{2\Delta_0}}{\sqrt{mK}}
\end{align}
Setting $\bar{w}_K = \frac{1}{K} \sum_{k=0}^{K-1} w_k$ and using Jensen's inequality on the convex function $f$, we get
\begin{align}
    \E[f(\bar{w}_K) - f^*] \leq \frac{\tau^2(1 + \sqrt{5}) + 2\tau\sqrt{\nicefrac{b}{L_{\text{max}}}}}{m}  + \frac{\tau\sqrt{2\Delta_0}}{\sqrt{mK}}
\end{align}
which concludes the proof.
\end{proof}


\section{Proof of Theorem~\ref{thm:adasvrg-main}}
\label{app:thm:adasvrg-main}
For the remainder of the appendix we define $\rho' \coloneqq \big( \frac{D^2}{\eta_{\text{min}}} + 2\eta_{\text{max}}\big) \sqrt{dL_{\text{max}}}$. We restate and prove~\cref{thm:adasvrg-main} for all three variants of \adasvrg/.

\begin{thmbox}
\begin{restatable}[\adasvrg/ with fixed-size inner-loop]{theorem}{restateAdaSVRGmain-general}
\label{thm:adasvrg-main-general}
Under the same assumptions as~\cref{lemma:adasvrg-outerloop}, \adasvrg/ with (a) step-sizes $\etak \in [\eta_{\text{min}}, \eta_{\text{max}}]$, (b) inner-loop size $m_k = n$ for all $k$, results in the following convergence rate after $K \leq n$ iterations. For the scalar variant,
\begin{align*}
    \E[f(\bar{w}_K) - f^*] \leq \frac{\rho^2 ( 1+ \sqrt{5}) + \rho\sqrt{2\left(f(w_0) - f^*\right)}}{K}
\end{align*}
and for the full matrix and diagonal variants,
\begin{align*}
    \E[f(\bar{w}_K) - f^*] \leq \frac{(\rho')^2 (1+ \sqrt{5})  + 2\rho'\sqrt{\nicefrac{d\delta}{L_{\text{max}}}}+ \rho'\sqrt{2\left(f(w_0) - f^*\right)}}{K},
\end{align*}
where $\bar{w}_K = \frac{1}{K} \sum_{k = 1}^{K} w_{k}$. 
\end{restatable}
\end{thmbox}
\begin{proof}
We have $\frac{1}{m} = \frac{1}{n} \leq \frac{1}{K}$ by the assumption. Using the result of~\cref{apdx:prop:adasvrg} for the scalar variant we have
\begin{align}
    \E[ f(\bar{w}_K) - f^*] &\leq  \frac{\rho^2(1 + \sqrt{5})}{m}  + \frac{\rho\sqrt{2\Delta_0}}{\sqrt{mK}} \\
    &\leq  \frac{\rho^2(1 + \sqrt{5}) }{K}  + \frac{\rho\sqrt{2\Delta_0}}{K}
\end{align}
Using the result of~\cref{apdx:prop:adasvrg} for the full matrix and diagonal variants we have
\begin{align}
    \E[ f(\bar{w}_K) - f^*] &\leq  \frac{(\rho')^2(1 + \sqrt{5}) + 2\rho'\sqrt{\nicefrac{d\delta}{L_{\text{max}}}}}{m}  + \frac{\rho'\sqrt{2\Delta_0}}{\sqrt{mK}} \\
    &\leq  \frac{(\rho')^2(1 + \sqrt{5}) + 2\rho'\sqrt{\nicefrac{d\delta}{L_{\text{max}}}}}{K}  + \frac{\rho'\sqrt{2\Delta_0}}{K}
\end{align}

\end{proof}


\begin{corollary}
Under the assumptions of~\cref{thm:adasvrg-main}, the computational complexity of AdaSVRG to reach $\epsilon$-accuracy is $O\left( \frac{n}{\epsilon}\right)$ when $\epsilon \geq \frac{\rho^2( 1 +\sqrt{5}) + \rho\sqrt{2\left(f(w_0) - f^*\right)}}{n}$ for the scalar variant and when $\epsilon \geq \frac{(\rho')^2( 1 +\sqrt{5}) + 2\rho' \sqrt{\nicefrac{d\delta}{L_{\text{max}}}} + \rho'\sqrt{2\left(f(w_0) - f^*\right)}}{n}$ for the full matrix and diagonal variants.
\end{corollary}

\begin{proof}
We deal with the scalar variant first. Let $c = \rho^2( 1 +\sqrt{5}) + \rho\sqrt{2\left(f(w_0) - f^*\right)}$. By the previous theorem, to reach $\epsilon$-accuracy we require
\begin{align}
    \frac{c}{K} \leq \epsilon
\end{align}
We thus require $O\left(\frac{1}{\epsilon}\right)$ outer loops to reach $\epsilon$-accuracy.
For $m_k = n$, $3n$ gradients are computed in each outer loop, thus the computational complexity is indeed $O\left(\frac{n}{\epsilon}\right)$.

The condition $\epsilon \geq \frac{c}{n}$ follows from the assumption that $K \leq n$.

The proof for the full matrix and diagonal variants is similar by taking $c = (\rho')^2( 1 +\sqrt{5}) + 2\rho' \sqrt{\nicefrac{d\delta}{L_{\text{max}}}} + \rho'\sqrt{2\left(f(w_0) - f^*\right)}$.
\end{proof}

\section{Proof of Theorem~\ref{thm:adasvrg-doubling}}
\label{app:thm:adasvrg-doubling}
\myquote{\begin{thmbox}\restateAdaSVRGdoubling*
\textit{With the same assumptions as above, the full matrix and diagonal variants of multi-stage AdaSVRG require $O(n \log \frac{1}{\epsilon} + \frac{1}{\epsilon})$ gradient evaluations to reach a $\left((\rho')^2 (1+\sqrt{5}) + \rho' \sqrt{\nicefrac{ d\delta}{L_{\text{max}}}}\right) \epsilon$-sub-optimality}
\end{thmbox}}
\begin{proof}
We deal with the scalar variant first. Let $\Delta_i := \E[f(\bar{w}^i) - f^*]$ and $c:= \rho^2(1 + \sqrt{5})$. Suppose that $K \geq 3$ and $m_i = 2^{i+1}$, as in the theorem statement. We claim that for all $i$, we have $\Delta_i \leq \frac{c}{2^i}$. We prove this by induction. For $i=0$, we have
\begin{align}
    \frac{c}{2^0} &= c = \rho^2(1 + \sqrt{5})\\
    &= L_{\text{max}}(1+\sqrt{5})\left(\frac{D^2}{\eta} + 2\eta\right)^2
\end{align}
The quantity $\frac{D^2}{\eta} + 2\eta$ reaches a minimum for $\eta^* = \frac{D}{\sqrt{2}}$. Therefore we can write
\begin{align}
    \frac{c}{2^0} &= L_{\text{max}}(1+\sqrt{5})\left(\sqrt{2}D+ 2\frac{D}{\sqrt{2}}\right)^2\\
    &= L_{\text{max}}(1+\sqrt{5}) 8D^2\\
    &\geq (1+\sqrt{5})8L_{\text{max}} \norm{\bar{w}_0 - w^*}^2\\
    &\geq (1+\sqrt{5})4 \left( f(\bar{w}_0) - f^*\right)\tag{smoothness}\\
    &= (1+\sqrt{5})4\Delta_0\\
    &\geq \Delta_0
\end{align}
Now suppose that $\Delta_{i-1} \leq \frac{c}{2^{i-1}}$ for some $i\geq 1$. Using the upper-bound analysis of AdaSVRG in~\cref{apdx:prop:adasvrg} we get
\begin{align}
    \Delta_i &\leq\frac{c}{m^i} + \frac{\rho\sqrt{2\Delta_{i-1}}}{\sqrt{m^i K}}\\
    &=\frac{c}{2^{i+1}} + \frac{\rho\sqrt{2\Delta_{i-1}}}{\sqrt{2^{i+1} K}} \\
    &\leq \frac{c}{2^{i+1}} + \frac{\rho \sqrt{2\frac{c}{2^{i-1}}}}{\sqrt{2^{i+1} K}} \tag{induction hypothesis}\\
    &= \frac{c}{2^{i+1}} + \frac{c}{2^{i+1}}\left(\frac{\rho}{\sqrt{c}}\sqrt{\frac{2^{i+1}}{2^{i-2}K}} \right)\\
    &\leq \frac{c}{2^{i+1}} + \frac{c}{2^{i+1}} \frac{\rho}{\sqrt{\rho^2(1 + \sqrt{5})}} \sqrt{\frac{8}{ K}}\\
    &= \frac{c}{2^{i+1}} + \frac{c}{2^{i+1}} \sqrt{\frac{8}{(1 + \sqrt{5})K}}
\end{align}
Since $K \geq 3$, one can check that $\frac{8}{(1+ \sqrt{5}) K} \leq 1$ and thus
\begin{align}
    \Delta_i \leq \frac{c}{2^{i+1}} + \frac{c}{2^{i+1}} = \frac{c}{2^i}
\end{align}
which concludes the induction step.\\
At time step $I = \log\frac{1}{\epsilon}$, we thus have $\Delta_I \leq \frac{c}{2^{I}} = c\epsilon$. All that is left is to compute the gradient complexity. If we assume that $K = \gamma$ for some constant $\gamma \geq 3$, the gradient complexity is given by
\begin{align*}
    \sum_{i=1}^I K ( n+ m^i) &= \sum_{i=1}^I \gamma (n + 2^{i+1})\\
    &= \gamma n\log\left(\frac{1}{\epsilon}\right) + \gamma \sum_{i=1}^I 2^{i+1} \\
    &\leq \gamma n\log\left(\frac{1}{\epsilon}\right) + 4\gamma \frac{1}{\epsilon}\\
    &= O\left( (n + \frac{1}{\epsilon}) \log\frac{1}{\epsilon}\right)
\end{align*}
which concludes the proof for the scalar variant.\\ 
Now we look at the full matrix and diagonal variants. Let's take $c = (\rho')^2(1 + \sqrt{5}) + 2\rho'\sqrt{\nicefrac{d \delta}{L_{\text{max}}}}$.
Suppose that $K \geq 3$ and $m_i = 2^{i+1}$, as in the theorem statement. Again, we claim that for all $i$, we have $\Delta_i \leq \frac{c}{2^i}$. We prove this by induction. For $i=0$, we have
\begin{align}
    \frac{c}{2^0} &= c = (\rho')^2(1 + \sqrt{5}) + 2\rho'\sqrt{\nicefrac{d\delta}{L_{\text{max}}}}\\
    &\geq dL_{\text{max}}(1+\sqrt{5})\left(\frac{D^2}{\eta} + 2\eta\right)^2
\end{align}
The quantity $\frac{D^2}{\eta} + 2\eta$ reaches a minimum for $\eta^* = \frac{D}{\sqrt{2}}$. Therefore we can write
\begin{align}
    \frac{c}{2^0} &\geq dL_{\text{max}}(1+\sqrt{5})\left(\sqrt{2}D+ 2\frac{D}{\sqrt{2}}\right)^2\\
    &= dL_{\text{max}}(1+\sqrt{5}) 8D^2\\
    &\geq d(1+\sqrt{5})8L_{\text{max}} \norm{\bar{w}_0 - w^*}^2\\
    &\geq d(1+\sqrt{5})4 \left( f(\bar{w}_0) - f^*\right)\tag{smoothness}\\
    &= d(1+\sqrt{5})4\Delta_0\\
    &\geq \Delta_0
\end{align}
The induction step is exactly the same as in the scalar case.
\end{proof}

\section{Proof of Theorem~\ref{thm:adasvrg-at}}
\label{app:thm:adasvrg-at}
We restate and prove~\cref{thm:adasvrg-at} for the three variants of \adasvrg/.
\begin{thmbox}\begin{restatable}[\adasvrg/ with adaptive-sized inner-loops]{theorem}{restateAdaSVRGadaptivetermination-general}
Under the same assumptions as~\cref{lemma:adasvrg-outerloop}, \adasvrg/ with (a) step-sizes $\etak \in [\eta_{\text{min}}, \eta_{\text{max}}]$, (b1) inner-loop size $m_k = \frac{\nu}{\epsilon}$ for all $k$ or (b2) inner-loop size $m_k = \frac{\nu}{f(w_k) - f^*}$ for outer-loop $k$, results in the following convergence rate,
\begin{align*}
\E[f(w_K) - f^*] \leq (3/4)^K [f(w_0) - f^*]. 
\end{align*}
where $\nu = 4\rho^2$ for the scalar variant and $\nu = \left( \frac{2\rho' + \sqrt{16 (\rho')^2 + 12\rho' \sqrt{\nicefrac{d\delta}{4L_{\text{max}}}}}}{3}\right)^2$ for the full matrix and diagonal variants.
\label{thm:adasvrg-at-general}
\end{restatable}
\end{thmbox}

\begin{proof}
Let us define
\begin{align*}
    \alpha = \begin{cases}
    \frac{1}{2}\big( \frac{D^2}{\eta_{\text{min}}} + 2\eta_{\text{max}}\big) &\text{(scalar variant)}\\
    \frac{1}{2}\big( \frac{D^2}{\eta_{\text{min}}} + 2\eta_{\text{max}}\big)\sqrt{d} &\text{(full matrix and diagonal variants)}
    \end{cases}
\end{align*}
and
\begin{align*}
    b = \begin{cases}
    0 &\text{(scalar variant)}\\
    d\delta &\text{(full matrix and diagonal variants)}
    \end{cases}
\end{align*}
Similar to the proof of Lemma~\ref{lemma:adasvrg-outerloop}, for a inner-loop $k$ with $m_k$ iterations and $\alpha := \frac{1}{2}\big( \frac{D^2}{\eta_{\text{min}}} + 2\eta_{\text{max}}\big)$ we can show 
\begin{align}
    \mathbb E \left [\sum_{t=1}^{m_k} f(\xt)-\f \right ]& \leq \alpha \sqrt{ \sum_{t=1}^{m_k}\mathbb{E}\left[ \norm{g_t}^2\right]+  b} \leq \alpha \sqrt{4L_{\text{max}}  \sum_{t=1}^{m_k} \mathbb{E}\left[f(\xt)-\f\right] + 4L_{\text{max}}  \sum_{t=1}^{m_k} \underbrace{\mathbb{E}\left[f(w_k) - \f\right]}_{\epsilon_k} +b}\\
    &\leq \alpha \sqrt{4L_{\text{max}} \mathbb{E}\left[ \sum_{t=1}^{m_k} f(\xt)-\f\right] + 4L_{\text{max}} m_k \epsilon_k +b}
\end{align}
Using~\cref{lem:quad_ineq} we get
\begin{align}
    \label{eq:adasvrg_subopt_upprbond-2}
    \mathbb E \left [\sum_{t=1}^{m_k} f(\xt)-\f \right ]& \leq 4L_{\text{max}} \alpha^2 + \sqrt{4L_{\text{max}}\alpha^2m_k \epsilon_k + \alpha^2 b}
\end{align}
If we set $m_k = \frac{C}{\epsilon_k}$, 
\begin{align}
    \label{eq:adasvrg_subopt_upprbond}
    \mathbb E \left [\sum_{t=1}^{m_k} f_t(\xt)-\f \right ]& \leq 4L_{\text{max}} \alpha^2 + \sqrt{4L_{\text{max}} \alpha^2 C+\alpha^2 b}\\
    & \leq 4L_{\text{max}} \alpha^2 + \sqrt{4L_{\text{max}} \alpha^2 C}+\sqrt{ \alpha^2 b}.
\end{align}
Define $a \coloneqq \sqrt{4L_{\text{max}}\alpha^2}$ and $\gamma \coloneqq \sqrt{b/4L_{\text{max}}}$, by dividing both sides of~\eqref{eq:adasvrg_subopt_upprbond} by $m_k$ and using the definition of $\xkk$, 
\begin{equation}
    \mathbb E \left [f(w_{k+1})-\f \right ] \leq \frac{a^2 + a(\sqrt{C}+\gamma)}{m_k} =  \left(\frac{a^2 + a (\sqrt{C}+\gamma)}{C} \right) \mathbb E \left [f(w_{k})-\f \right ] 
\end{equation}
Setting $C=\left(\frac{2a+\sqrt{4a^2+12(a^2+a\gamma)}}{3}\right)^2$, we get $\epsilon_{k+1} \leq 3/4 \epsilon_k$. However, the above proof requires knowing $\epsilon_{k}$. Instead, let us assume a target error of $\epsilon$, implying that we want to have $\mathbb E [f(w_{K}) - \f ] \leq \epsilon$. Going back to~\eqref{eq:adasvrg_subopt_upprbond-2} and setting $m_k = C/\epsilon$, we obtain, 
\begin{align}
    \mathbb E \left [f(w_{k+1})-\f \right ] & \leq \frac{a^2 + a(\sqrt{C\epsilon_k/\epsilon}+\gamma)}{m_k}
          = \left(\frac{a^2 + a(\sqrt{C\epsilon_k/\epsilon}+\gamma)}{C} \right) \epsilon \\
         & \leq  \frac{a^2 \epsilon_k + a\sqrt{C\epsilon_k/\epsilon} \sqrt{\epsilon_k \epsilon}+a\gamma\epsilon_k}{C} \tag{Assuming that $\epsilon \leq \epsilon_{k}$ for all $k \leq K$.}\\
         & = \left(\frac{a^2 + a\sqrt{C}+a\gamma}{C} \right) \epsilon_k
         = \left(\frac{a^2 + a\sqrt{C}+a\gamma}{C} \right) \mathbb E \left [f(w_{k})-\f \right ]
\end{align}
With $C=\left(\frac{2a+\sqrt{4a^2+12(a^2+a\gamma)}}{3}\right)^2$, we get linear convergence to $\epsilon$-suboptimality. Based on the above we require $K = \mathcal O(\log(1/\epsilon))$ outer-loops. However in each outer-loop we need $\mathcal O(n+\frac{1}{\epsilon})$ gradient evaluations. All in all, our total computation complexity is of $\mathcal{O} ((n+\frac{1}{\epsilon}) \log(\frac{1}{\epsilon}))$.\\
The proof is done by noticing that in the scalar variant, $\gamma = 0$ and $a = \rho$ so that
\begin{align*}
    C = \left( \frac{2\rho + \sqrt{4\rho^2 + 12 \rho^2}}{3} \right)^2 = 4\rho^2
\end{align*}
and in the full matrix and diagonal variants, $a = \rho'$ and $\gamma = \sqrt{\nicefrac{d\delta}{4L_{\text{max}}}}$ so that
\begin{align*}
    C = \left( \frac{2\rho' + \sqrt{16(\rho')^2 + 12\rho' \sqrt{\nicefrac{d\delta}{4L_{\text{max}}}}}}{3} \right)^2
\end{align*}
\end{proof}


\newpage
\section{Proof of Theorem~\ref{thm:phase_trns}}
\label{app:thm:phase_trns} 
We restate and prove~\cref{thm:phase_trns} for the three variants of \adagrad/.

\begin{thmbox}\begin{restatable}[Phase Transition in \adagrad/ Dynamics]{theorem}{restateAdaGradPhaseTrans-general}
\label{thm:phase_trns-general}
Under the same assumptions as~\cref{lemma:adasvrg-outerloop} and (iv) $\sigma^2$-bounded stochastic gradient variance and defining $T_0 = \frac{\rho^2 L_{\text{max}}}{\sigma^2}$, for constant step-size \adagrad/ we have 
$\E \|G_t\|_{*} = O(1)  \text{ for } t \leq T_0$, and $\E \|G_t\|_{*} = O (\sqrt{t-T_0})  \text{ for } t \geq T_0$.\\
The same result holds for the full matrix and diagonal variants of constant step-size AdaGrad for $T_0 = \frac{\left(\rho' \sqrt{L_{\text{max}}} + \sqrt{2\rho' \sqrt{d\delta L_{\text{max}}} + d\delta }\right)^2}{\sigma^2}$
\end{restatable}
\end{thmbox}

\begin{proof}
We start with the scalar variant. Consider the general \adagrad/ update 
\begin{equation}
    \xtp = \proj{X, A_t}{\xt - \lr A_{t}^{-1} \nabla \fit(\xt)}
\end{equation}
The same we did in the proof of \cref{thm:adasvrg-main}, we can bound suboptimality as
\begin{align}
& \| \xtp - \xopt\|_{A_t}^2 \leq \| \xt - \xopt\|_{A_t}^2 - 2\lr \dpr{\xt-\xopt}{\nabla \fit(\xt)} + \lr^2 \|\nabla \fit(\xt)\|_{A_t^{-1}}^2  
\end{align}
By re-arranging, dividing by $\lr$ and summing for $T$ iteration we have 
\begin{align}
   \sum_{t=1}^T \dpr{\xt-\xopt}{\nabla \fit(\xt)} 
    &\leq \frac{1}{2\lr}\sum_{t=1}^T  \|\xt-\xopt\|^2_{A_t-A_{t-1}} + \frac{\lr}{2} \sum_{t=1}^T \|\nabla \fit(\xt)\|_{A_t^{-1}}^2  \tag{\cref{lemma:adasvrg-outerloop}} \\
    & \leq \frac{1}{2}\left( \frac{D^2}{\lr} + 2\lr \right) \Tr{A_T}\label{ineq:phase-transition1}
\end{align}
Define
\begin{align*}
    \alpha = \begin{cases}
    \frac{1}{2}\big( \frac{D^2}{\eta} + 2\eta\big) & \text{(scalar variant)}\\
    \frac{1}{2}\big( \frac{D^2}{\eta} + 2\eta\big)\sqrt{d} &\text{(full matrix and diagonal variants)}
    \end{cases}
\end{align*}
and
\begin{align*}
    b = \begin{cases}
    0 & \text{(scalar variant)}\\
    d \delta  &\text{(full matrix and diagonal variants)}
    \end{cases}
\end{align*}
We then have $\Tr{A_T} \leq \alpha \sqrt{\sum_{t=1}^T\|\nabla \fit(\xt)\|^2 + b }$ by~\cref{lem:telescoping} and~\cref{lem:trace-ineq}. Going back to~\cref{ineq:phase-transition1} and taking expectation and using the upper-bound we get
\begin{align}
\E \left[\sum_{t=1}^T \dpr{\xt-\xopt}{\nabla f(\xt)}\right] & \leq \alpha \E \left[\sqrt{\sum_{t=1}^T \|\nabla \fit(\xt)\|^2 + b } \right]
\end{align}
Using convexity of $f$ and Jensen's inequality on the (concave) square root function, we have 
\begin{align}
   \mathbb E \left [\sum_{t=1}^T f(\xt)-\f \right ]  &\leq \alpha \sqrt{\sum_{t=1}^T \mathbb E\|\nabla \fit(\xt)\|^2 + b}\\
   &= \alpha \sqrt{\sum_{t=1}^T \E \left[ \|\nabla f_{i_t}(x_t) - \nabla f(x_t) \|^2 + 2 \langle \nabla f_{i_1}(x_t) - \nabla f(x_t), \nabla f(x_t)\rangle + \| \nabla f(x_t) \|^2 \right] + b} \label{apdx:eq-phase-transition3}\\
   &= \alpha\sqrt{\sum_{t=1}^T \E \| \nabla f_{i_t}(x_t) - \nabla f(x_t)\|^2 + \E \| \nabla  f(x_t)\|^2 + b} \tag{since $\E[\nabla f_{i_t}(x_t) - \nabla f(x_t)] = 0$}\\
   &\leq \alpha\sqrt{\sum_{t=1}^T\left( \E \| \nabla f_{i_t}(x_t) - \nabla f(x_t)\|^2 + 2L_{\text{max}}\E\left(f(x_t) - f^*\right)\right) + b } \label{apdx:eq-phase-transition}
\end{align}
where we used smoothness in the last inequality. Now, if $\sigma = 0$, namely $\nabla f(x_t) = \nabla f_{i_t}(x_t)$, we have
\begin{align}
   \left [\sum_{t=1}^T f(\xt)-\f \right ]  &\leq \alpha\sqrt{2 L_{\text{max}}\sum_{t=1}^T  (f(x_t) - f^*) + b}
\end{align}
Using~\cref{lem:quad_ineq} we get
\begin{align}
    \sum_{t=1}^T (f(x_t) - f^*) \leq \alpha^2 2L_{\text{max}} + \sqrt{\alpha^2 b}
\end{align}
so that
\begin{align}
    \sqrt{\sum_{t=1}^T \left(f(\xt)-\f\right)  }  \leq \alpha\sqrt{2L_{\text{max}}} + \sqrt{\alpha}b^{1/4}
\end{align}
Now,
\begin{align}
    \| G_T\|_* &= \sqrt{\Tr{ \sum_{t=1}^T \nabla f(x_t)^T \nabla f(x_t) + \delta}} = \sqrt{\sum_{t=1}^{T} \norm{\nabla f(x_t)}^2 + b}
\end{align}
Thus we have
\begin{align}
    \| G_T\|_* &=  
    \sqrt{\sum_{t=1}^T \| \nabla f(x_t) \|^2 + b } \leq  \sqrt{b } + \sqrt{2L_{\text{max}} \sum_{t=1}^T \left(f(x_t) - f^*\right)} \leq \sqrt{b} + \sqrt{2\alpha L_{\text{max}}} b^{1/4} + 2\alpha L_{\text{max}}
\end{align}
where we used smoothness for the inequality. This shows that $\|G_T\|_*$ is a bounded series in the deterministic case.
Now, if $\sigma \not = 0$, going back to~\cref{apdx:eq-phase-transition} we have
\begin{align}
    \sum_{t=1}^T \E \left[f(x_t) - \f \right] &\leq \alpha\sqrt{ T \sigma^2 +2L_{\text{max}}\sum_{t=1}^T\left( \E\left(f(x_t) - f^*\right)\right) + b}
\end{align}
Using~\cref{lem:quad_ineq} we get
\begin{align}
    \sum_{t=1}^T \E\left[f(x_t) - \f\right] &\leq 2\alpha^2L_{\text{max}} + \sqrt{T\alpha^2 \sigma^2 + \alpha^2 b} \leq 2\alpha^2L_{\text{max}} + \sqrt{T\alpha^2 \sigma^2} + \alpha\sqrt{ b}\label{apdx:eq-phase-transition4}
\end{align}
We then have
\begin{align}
    \E \|G_T\|_*^2 &= \E\left[\sum_{t=1}^T\|\nabla f_{i_t}(x_t) \|^2 + b \right]\\
    &= \sum_{t=1}^T \E\| \nabla f_{i_t} (x_t) - \nabla f(x_t) \|^2 + \sum_{t=1}^T \E\| \nabla f(x_t) \|^2 + b \tag{same as~\cref{apdx:eq-phase-transition3}}\\
    &\leq T\sigma^2 + 2L_{\text{max}}\sum_{t=1}^T \E\left[ f(x_t) - \f\right] + b \tag{smoothness}\\
    &\leq T\sigma^2 + 4\alpha^2 L_{\text{max}}^2 + 2L_{\text{max}} \sqrt{T\alpha^2 \sigma^2} + 2 L_{\text{max}}\alpha \sqrt{b} + b \tag{by~\cref{apdx:eq-phase-transition4}}\\
    &\leq (\sigma \sqrt{T} + 2\alpha L_{\text{max}})^2 + 2L_{\text{max}}\alpha \sqrt{b} + b \label{apdx:eq-phase-transition2}
\end{align}
from which we get
\begin{align}
    \E \|G_T\|_* &= \E \sqrt{\|G_T\|_*^2}\\
    &\leq \sqrt{\E \|G_T\|_*^2} \tag{Jensen's inequality}\\
    &\leq \sqrt{(\sigma \sqrt{T} + 2\alpha L_{\text{max}})^2 + 2 L_{\text{max}}\alpha\sqrt{b} + b} \tag{by~\cref{apdx:eq-phase-transition2}}\\
    &\leq \sigma \sqrt{T} + 2 \alpha L_{\text{max}} + \sqrt{2L_{\text{max}} \alpha \sqrt{b} + b} 
\end{align}
This implies that for $T \leq \frac{\left( 2 \alpha L_{\text{max}} + \sqrt{2L_{\text{max}} \alpha \sqrt{b} + b}\right)^2}{\sigma^2}$, we have
\begin{align}
    \E \| G_T \|_* \leq 4 \alpha L_{\text{max}} + 2\sqrt{2L_{\text{max}} \alpha \sqrt{b} + b} 
\end{align}
and for $ T\geq \frac{\left( 2 \alpha L_{\text{max}} + \sqrt{2L_{\text{max}} \alpha \sqrt{b} + b}\right)^2}{\sigma^2}$,
\begin{align}
    \E \| G_T\|_* = O(\sqrt{T})
\end{align}
The proof is done by noticing that
\begin{align*}
    \alpha = \begin{cases}
    \frac{\rho}{\sqrt{L_{\text{max}}}} &\text{(scalar case)}\\
    \frac{\rho'}{\sqrt{L_{\text{max}}}} &\text{(full matrix and diagonal cases)}
    \end{cases}
\end{align*}
\end{proof}

\begin{corollary}
\label{cor:phase_trns}
Under the same assumptions as~\cref{lemma:adasvrg-outerloop}, for outer-loop $k$ of \adasvrg/ with constant step-size $\eta_k$, there exists $T_0 = \frac{C}{f(w_k) - f^*}$ such that,
\[
\E \|G_t\|_{*} =   
\begin{cases}
O(1), & \text{for } t \leq T_0 \\
O (\sqrt{t-T_0}) , & \text{for } t \geq T_0. \\
\end{cases} 
\]
\end{corollary}
\begin{proof}
Using~\cref{thm:phase_trns-general} with $\sigma^2 = f(\xk) - f^*$ gives us the result. 
\end{proof}
\newpage
\section{Helper Lemmas}
\label{app:helper-lemmas}
We make use of the following helper lemmas from~\citep{vaswani2020adaptive}, proved here for completeness. 
\begin{lemma}
For any of the full matrix, diagonal and scalar versions, we have
\begin{align*}
    \sum_{t=1}^m \norm{x_t - w^*}^2_{A_{t} - A_{t-1}} \leq D^2 \mathrm{Tr}(A_m)
\end{align*}
\label{lem:telescoping}
\end{lemma}
\begin{proof}
For any of the three versions, we have by construction that $A_t$ is non-decreasing, i.e. 
$A_t - A_{t-1} \succeq 0$ (for the scalar version, we consider $A_t$ as a matrix of dimension 1 for simplicity). We can then use the bounded feasible set assumption to get
\aligns{
    \textstyle
    \sum_{t=1}^m 
    \norm{\xt - \xopt}^2_{{A_t} -A_{t-1}} 
    &\leq 
    \textstyle
    \sum_{t=1}^m \lambda_{\max}\paren{{A_t} -A_{t-1}}\norm{\xt - \xopt}^2
    \\
    &\leq 
    \textstyle
    D^2\!\sum_{t=1}^T\lambda_{\max}\paren{{A_t} -A_{t-1}}.
\intertext{We then upper-bound $\lambda_{\max}$ by the trace and use the linearity of the trace to telescope the sum,}
    \textstyle
    &\leq 
    \textstyle
    D^2 \sum_{t=1}^m\Tr{{A_t} -A_{t-1}}
    = D^2 \sum_{t=1}^m\Tr{A_t} - \Tr{A_{t-1}} ,
    \\ 
    &= 
    \textstyle
    D^2 \paren{\Tr{A_m} - \Tr{A_0}}
    \leq D^2\Tr{A_m} 
}
\end{proof}

\begin{lemma} For any of the full matrix, diagonal and scalar versions, we have
\begin{align*}
    \sum_{t=1}^m \norm{g_t}^2_{A_t^{-1}} \leq 2\mathrm{Tr}(A_m)
\end{align*}
Moreover, for the scalar version we have
\begin{align*}
    \Tr{A_m} \leq \sqrt{ \sum_{t=1}^m \norm{g_t}^2}
\end{align*}
and for the full matrix and diagonal version we have
\begin{align*}
    \Tr{A_m} \leq \sqrt{d \sum_{t=1}^m \norm{g_t}^2 + d^2\delta}
\end{align*}
\label{lem:trace-ineq}
\end{lemma}
\begin{proof}
We prove this by induction. Start with $m=1$. 

For the full matrix version, $A_1=(\delta I+g_1g_1^\top)^{\nicefrac{1}{2}}$ and we have
\begin{align*}
    \normsq{g_1}_{A_1^{-1}}  
    & = g_1^\top A_1^{-1}g_1
    = \Tr{g_1^\top A_1^{-1}g_1}  
    = \Tr{A_1^{-1}g_1g_1^\top}\\
    &= \Tr{A_1^{-1}(A_1^2 - \delta I)} 
    = \Tr{A_1}-\Tr{ \delta A_1^{-1}} \leq \Tr{A_1}
\end{align*}

For the diagonal version $A_1=(\delta I+\diag(g_1g_1^\top))^{\nicefrac{1}{2}}$ we have
\begin{align}
    \norm{g_1}^2_{A_1^{-1}} = g_1^\top A_t^{-1} g_1 = \Tr{g_1^\top A_1^{-1} g_1} = \Tr{A_1^{-1} g_1 g_1^\top}
\end{align}
Since $A_1^{-1}$ is diagonal, the diagonal elements of $A_1^{-1} g_1 g_1^\top$ are the same as the diagonal elements of $A_1^{-1} \diag(g_1 g_1^\top)$. Thus we get
\begin{align*}
    \norm{g_1}^2_{A_1^{-1}} = \Tr{A_1^{-1} \diag(g_1 g_1^\top)} = \Tr{A_1^{-1} \left(A_1^2 - \delta I\right)} = \Tr{A_1} - \delta \Tr{A_1^{-1}} \leq \Tr{A_1}
\end{align*}
For the scalar version $A_1 = \left(  g_1^\top g_1\right)^{\nicefrac{1}{2}}$ and we have
\begin{align*}
    \norm{g_1}^2_{A_1^{-1}} = A_1^{-1} \norm{g_1}^2 = A_1^{-1} g_1^\top g_1 = A_1^{-1} A_1^2 = A_1 = \Tr{A_1} 
\end{align*}

\textbf{Induction step: }Suppose now that it holds for $m-1$, i.e. $\sum_{t=1}^{m-1} \norm{g_t}^2_{A_t^{-1}}\leq 2 \Tr{A_{m-1}}$.
We will show that it also holds for $m$. \\
For the full matrix version we have
\aligns{
    \hspace{-1em} \textstyle 
    \sum_{t=1}^m \norm{g_t }^2_{A_t^{-1}} 
    &\leq 2\Tr{A_{m-1}} + \normsq{g_m}_{A_m^{-1}} 
    \tag{Induction hypothesis}
    \\ 
    &= 2\Tr{(A_m^2-g_m g_m^\top)^{\nicefrac{1}{2}}} + \Tr{A_m^{-1}g_m g_m^\top} 
    \tag{AdaGrad update}
}
We then use the fact that for
any $X \succeq Y \succeq 0$, we have \citep[Lemma 8]{duchi2011adaptive}
\aligns{
    2\Tr{(X-Y)^{\nicefrac{1}{2}}}+\Tr{X^{-\nicefrac{1}{2}}Y}\leq 2\Tr{X^{\nicefrac{1}{2}}}.
}
As $X = A_m^2\succeq Y = g_m g_m^\top \succeq 0$, we can use the above inequality and the induction holds for $m$.

For the diagonal version we have
\begin{align*}
    \sum_{t=1}^m \norm{g_t}_{A_t^{-1}}^2 &\leq 2 \Tr{A_{m-1}} + \norm{g_m}^2_{A_{m}^{-1}} \tag{Induction hypothesis}\\
    &= 2\Tr{\left( A_m^2 - \text{diag}(g_m g_m^\top\right)^{1/2}} + \Tr{A_m^{-1} g_m g_m^\top} \tag{AdaGrad update}
\end{align*}
As before, since $A_m^{-1}$ is diagonal, we have that the diagonal elements of $A_m^{-1} g_m g_m^\top$ are the same as the diagonal elements $A_m^{-1} \diag(g_m g_m^\top)$. Thus we get
\begin{align*}
    \sum_{t=1}^m \norm{g_t}_{A_t^{-1}}^2 &\leq 2\Tr{\left( A_m^2 - \text{diag}(g_m g_m^\top\right)^{1/2}} + \Tr{A_m^{-1} \diag(g_m g_m^\top)}
\end{align*}
We can then again apply the result from \citet[Lemma 8]{duchi2011adaptive} with $X = A_m^2 \succeq Y = \diag(g_m g_m^\top) \succeq 0$, and we obtain the desired result.

For the scalar version, since $A_m^{-1}$ is a scalar we have
\begin{align*}
    \sum_{t=1}^m \norm{g_t}_{A_t^{-1}}^2 &\leq 2 \Tr{A_{m-1}} + \norm{g_m}^2_{A_{m}^{-1}} \tag{Induction hypothesis}\\
    &= 2\Tr{\left( A_m^2 - g_m^\top g_m\right)^{1/2}} + \Tr{ A_m^{-1}g_m^\top g_m} \tag{AdaGrad update}
\end{align*}
We can then again apply the result from \citet[Lemma 8]{duchi2011adaptive} with $X = A_m^2 \geq Y = g_m^\top g_m \geq 0$, and we obtain the desired result.
\\

\textbf{Bound on the trace: }For the trace bound, recall that $A_m= G_m^{1/2}$.
For the scalar version we have
\aligns{
\textstyle
    \Tr{A_m} = \Tr{G_m^{\nicefrac{1}{2}}} = G_m^{\nicefrac{1}{2}} = \sqrt{\sum_{t=1}^m g_t^\top g_t } = \sqrt{\sum_{t=1}^m \normsq{g_t} }
}

For the diagonal and full matrix variants, we use Jensen's inequality to get
\aligns{
    \Tr{A_m} = \Tr{G_m^{\nicefrac{1}{2}}}
    &= \sum_{j=1}^d \sqrt{\lambda_j(G_m)}
    = d\bigg(\frac{1}{d}\sum_{j=1}^d \sqrt{\lambda_j(G_m)} \bigg)
    \\
    &\leq d\sqrt{\frac{1}{d}\sum_{j=1}^d\lambda_j(G_m)}
    = \sqrt{d}\sqrt{\Tr{G_m}}.
}
there $\lambda_j(G_m)$ denotes the $j$-th eigenvalue of $G_m$.\\
For the full matrix version, we have
\aligns{
    \textstyle 
    \sqrt{\Tr{G_m}}
    = \sqrt{\Tr{\sum_{t=1}^m g_t g_t^\top +\delta I}} 
    = \sqrt{\sum_{t=1}^m\Tr{g_t g_t^\top}+ d\delta }
    = \sqrt{\sum_{t=1}^m\normsq{g_t}+d\delta}
}

For the diagonal version, we have
\aligns{
    \textstyle 
    \sqrt{\Tr{G_m}}
    = \sqrt{\Tr{\sum_{t=1}^m \diag(g_t g_t^\top) + \delta I}} 
    = \sqrt{\sum_{t=1}^m\Tr{\diag(g_t g_t^\top)} + d\delta}
    = \sqrt{\sum_{t=1}^m\normsq{g_t} + d\delta} 
}
which concludes the proof.
\end{proof}
 
\begin{lemma}
\label{lem:quad_ineq}
If $x^2 \leq a(x+b)$ for $a\geq 0$ and $b \geq 0$, 
\begin{align*}
    x \leq \frac{1}{2}(\sqrt{a^2 + 4ab} + a) \leq a+ \sqrt{ab}
\end{align*}
\end{lemma}
\begin{proof}
The starting point is the quadratic inequality
$x^2 - ax - ab \leq 0$.
Letting $r_1 \leq r_2$ be the roots of the quadratic, the inequality holds if $x \in [r_1, r_2]$. 
The upper bound is then given by using $\sqrt{a+b} \leq \sqrt{a} +\sqrt{b}$
\aligns{
    r_2 
    = \frac{a + \sqrt{a^2 + 4ab}}{2}
    \leq \frac{a + \sqrt{a^2} + \sqrt{4ab}}{2}
    = a + \sqrt{ab}. 
    \tag*{\qedhere}
}
\end{proof}

\section{Counter-example for line-search for \svrg/}
\label{app:line-search-counter-example}

\begin{algorithm}[H]
\caption{SVRG-Inner-Armijo}
\label{alg:SVRG-Armijo}
initialization;\\
Input :$w^0 \in \mathbb{R}^n$, $\eta_{\text{max}} > 0$, $c>0$
\\
\For{$k = 0, 1, \dots$ }{
    $x^k_0 = w^k$\\
    \For{$t=0,\dots, m-1$}{
    $i_t \sim U[n]$ (pick $i_t$ uniformly from $\{1, \dots, n\}$)\\
    $g_t = \nabla f_{i_t}(x^k_t) - \nabla f_{i_t}(w^k) + \nabla f(w^k)$\\
    Find $\eta_t \leq \eta_{\text{max}}$ maximal such that\\
    $f_{i_t}(x_t^k - \eta_t g_t) \leq f_{i_t}(x_t^k) - c\eta_t \norm{g_t}^2$\\
    $x_{t+1}^k \leftarrow x_t^k - \eta_{t}g_t$
    }
    $w^{k+1} = x_m^k$
}
\end{algorithm}

\begin{proposition}
For any $c > 0, \eta_{\text{max}} > 0$, there exists a 1-dimensional function $f$ whose minimizer is $x^*=0$, and for which the following holds: If at any point of~\cref{alg:SVRG-Armijo}, we have $|x_t^k| \in \big(0, \min\{ \frac{1}{c}, 1\}\big)$, then $|x_{t+1}^k| \geq |x_t^k|$.
\end{proposition}
\begin{proof}
Define the following function
\begin{align}
    f(x) = \frac{1}{2}\big( f_1(x) + f_2(x)\big) =  \frac{1}{2} \big( a(x-1)^2 + a(x+1)^2\big) = a\big( x^2 + 1\big)
\end{align}
where $a>0$ is a constant that will be determined later. We then have the following
\begin{align*}
    &f'(x) = 2ax\\
    &f_1'(x) = 2a(x-1)\\
    &f_2'(x) = 2a(x+1)
\end{align*}
The minimizer of $f$ is 0, while the minimizers of $f_1$ and $f_2$ are 1 and -1, respectively. This symmetry will make the algorithm fail.\\
Now, as stated by the assumption, let $|x_t^k| \in \big(0, \min\{ \frac{1}{c}, 1\}\big)$. WLOG assume $x_t^k > 0$, the other case is symmetric.\\
\textbf{Case 1}: $i_t = 1$. Then we have
\begin{align*}
    g_t &= f_1'(x_t^k) - f_1'(w^k) + f'(w_k)\\
    &= 2a(x_t^k - 1) - 2a(w^k - 1) + 2aw^k\\
    &= 2ax_t^k > 0
\end{align*}
Observe that $g_t > 0$. Since $x_t^k < 1$ and the function $f_1$ is strictly decreasing in the interval $(-\infty, 1]$, moving in the direction $-g_t$ from $x_t^k$ can only increase the function value. Thus the Armijo line search will fail and yield $\eta_t = 0$. Thus in that case $x_{t+1}^k = x_t^k$.\\
\textbf{Case 2:} $i_t = 2$. Then we have
\begin{align*}
    g_t &= f_2'(x_t^k) - f_2'(w^k) + f'(w_k)\\
    &= 2a(x_t^k + 1) - 2a(w^k + 1) + 2aw^k\\
    &= 2ax_t^k > 0
\end{align*}
The Armijo line search then reads
\begin{align*}
    f_2(x_t^k - 2a\eta_t x_t^k) \leq f_2(x_t^k) - c\eta_t (2ax_t^k)^2
\end{align*}
which we can rewrite as
\begin{align*}
    &a\big( x_t^k + 1 - 2a\eta_t x_t^k\big)^2 \leq a(x_t^k + 1)^2 - 4ca^2\eta_t (x_t^k)^2\\
    \Rightarrow &a(x_t^k + 1)^2 - 4a^2\eta_t(x_t^k + 1)x_t^k + 4\eta_t^2 a^3 (x_t^k)^2 \leq a(x_t^k + 1)^2 - 4ca^2\eta_t (x_t^k)^2
\end{align*}
Simplifying this gives
\begin{align*}
    \eta_t a (x_t^k)^2 \leq (x_t^k + 1)x_t^k - c(x_t^k)^2
\end{align*}
which simplifies even further to
\begin{align*}
    \eta_t \leq \frac{1 + \frac{1}{x_t^k} - c}{a}
\end{align*}
Therefore, the Armijo line-search will return a step-size such that
\begin{align}
    \eta_t \geq \min \bigg\{ \frac{1 + \frac{1}{x_t^k} - c}{a}, \eta_{\text{max}}\bigg\} \label{apdx:counter-ex1}
\end{align}
Now, recall that by assumption we have $x_t^k < 1/c$. Then $1/x_t^k - c > 0$, which implies that
\begin{align*}
    \frac{1 + \frac{1}{x_t^k} - c}{a} \geq \frac{1}{a}
\end{align*}
Now is the time to choose $a$. Indeed, if $a$ is such that $1/a \leq \eta_{\text{max}}$, we then have by~\cref{apdx:counter-ex1} that
\begin{align*}
    \eta_t \geq \frac{1}{a}
\end{align*}
We then have
\begin{align*}
    x_{t+1}^k &= x_t^k - \eta_t g_t\\
    &= x_t^k -2a \eta_t x_t^k\\
    &= (1- 2a\eta_t)x_t^k\\
    &\leq (1 - 2)x_t^k = -x_t^k
\end{align*}
where the inequality comes from $\eta_t \geq 1/a$ and the fact that $x_t^k \geq 0$. Thus we indeed have $|x_{t+1}^k| \geq |x_t^k|$.
\end{proof}
\newpage
\section{Additional Experiments}
\label{app:additional-experiments}

\subsection{Poor performance of \adagrad/ compared to variance reduction methods}
\begin{figure}[H]

\centering
\includegraphics[scale = 0.3]{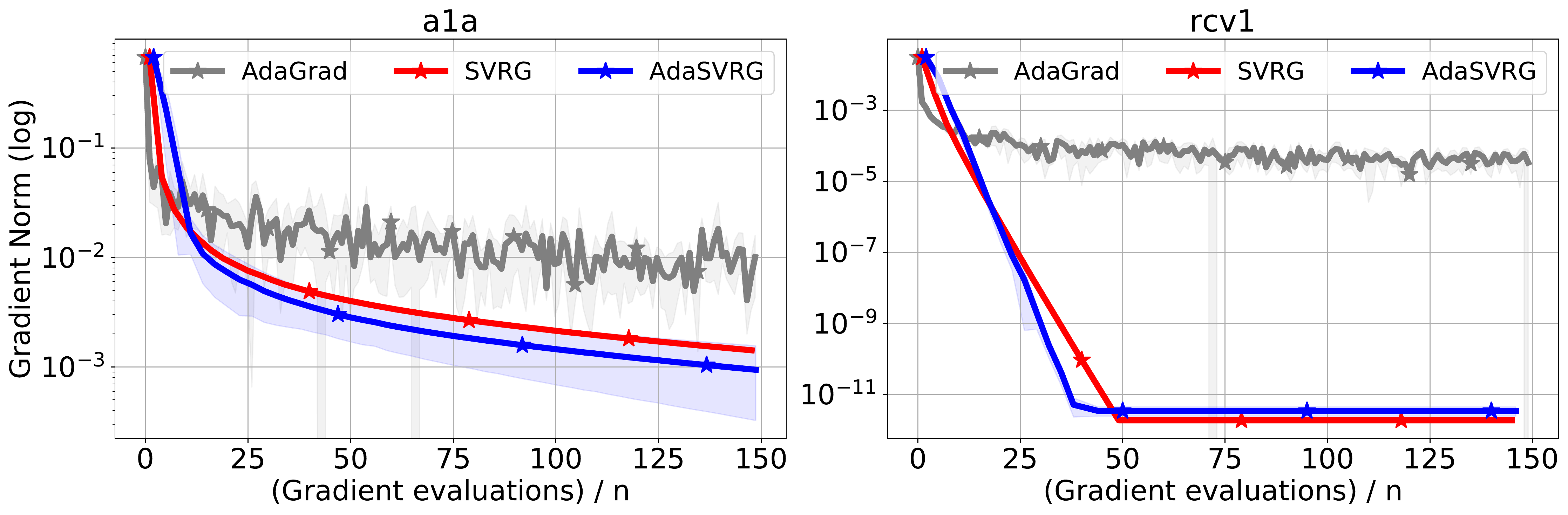}
\caption{Poor performance of best-tuned step-size (after grid-search) \adagrad/ compared to \svrg/ and \adasvrg/ on logistic regression problems with batch size 64 on two example datasets}
\label{fig:comparison-adagrad}
\end{figure}

\subsection{Additional experiments with batch-size = 64}
\vspace{-2ex}
\begin{figure}[H]
\centering
\subfigure[Comparison against best-tuned variants]{\label{fig:3datasets-logistic-epoch-bs=64}\includegraphics[scale = 0.3]{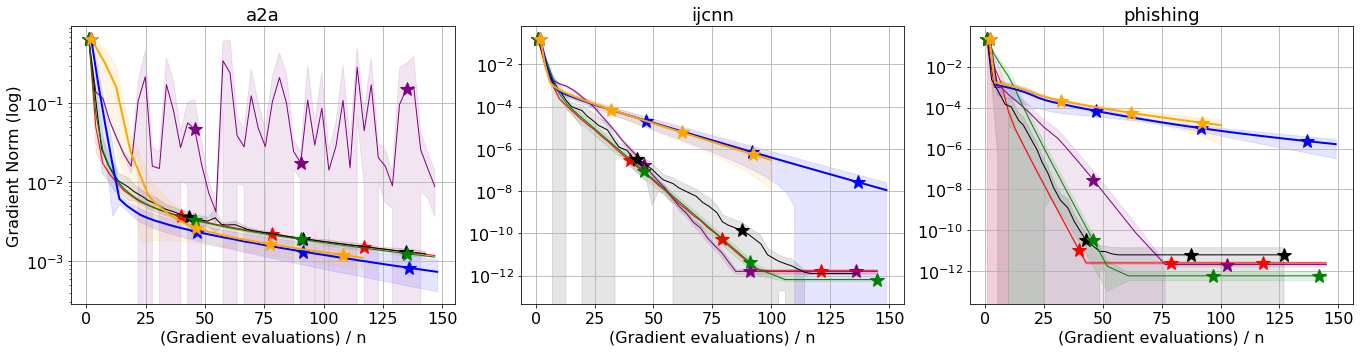}} \\
\subfigure[Sensitivity to step-size]{\label{fig:3datasets-logistic-stepsize-bs=644}\includegraphics[scale = 0.3]{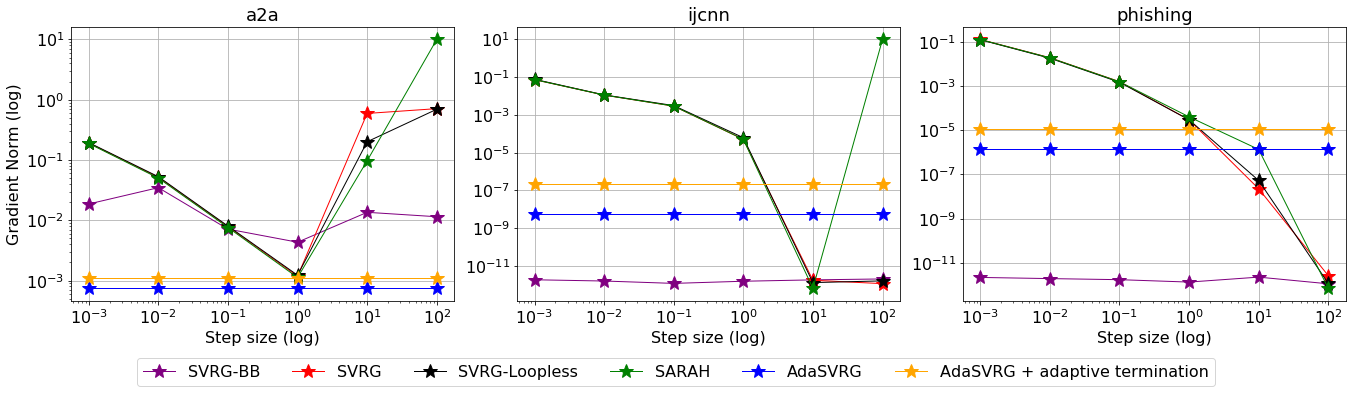}}
\caption{Comparison of AdaSVRG against \svrg/ variants, \svrg/-BB and \sarah/ for logistic loss and batch-size = $64$. For the sensitivity to step-size plot to be readable, we limit the gradient norm to a maximum value of 10.}
\end{figure}

\begin{figure}[H]
\centering     
\subfigure[Comparison against best-tuned variants]{\label{fig:3datasets-huber-epoch-bs=64}\includegraphics[scale = 0.3]{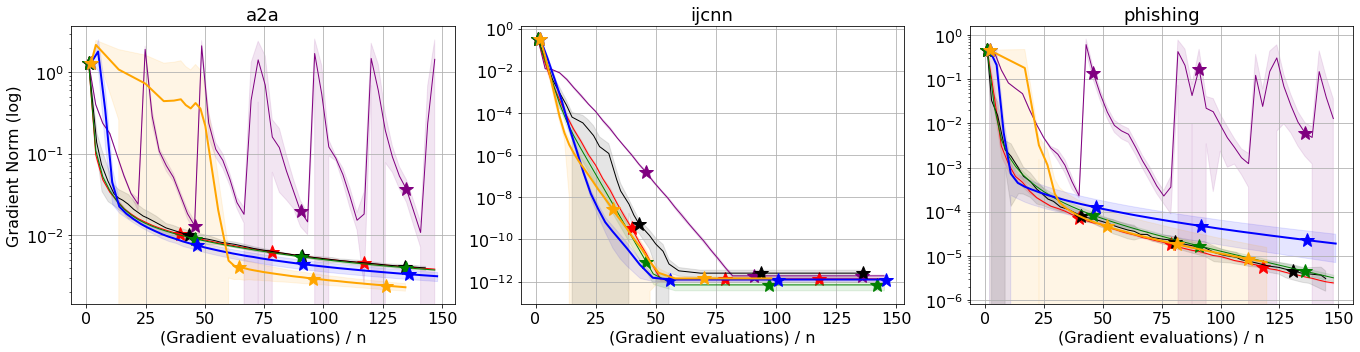}} \\
\subfigure[Sensitivity to step-size]{\label{fig:3datasets-huber-stepsize-bs=644}\includegraphics[scale = 0.3]{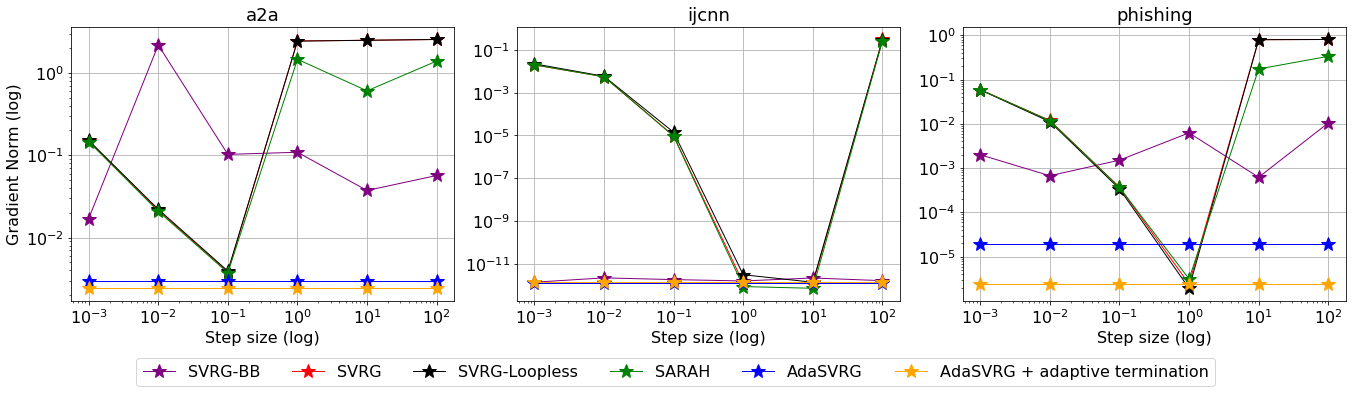}}
\caption{Comparison of AdaSVRG against \svrg/ variants, \svrg/-BB and \sarah/ for Huber loss and batch-size = $64$. For the sensitivity to step-size plot to be readable, we limit the gradient norm to a maximum value of 10.}
\end{figure}

\begin{figure}[H]
\centering     
\subfigure[Comparison against best-tuned variants]{\label{fig:3datasets-squared-epoch-bs=64}\includegraphics[scale = 0.3]{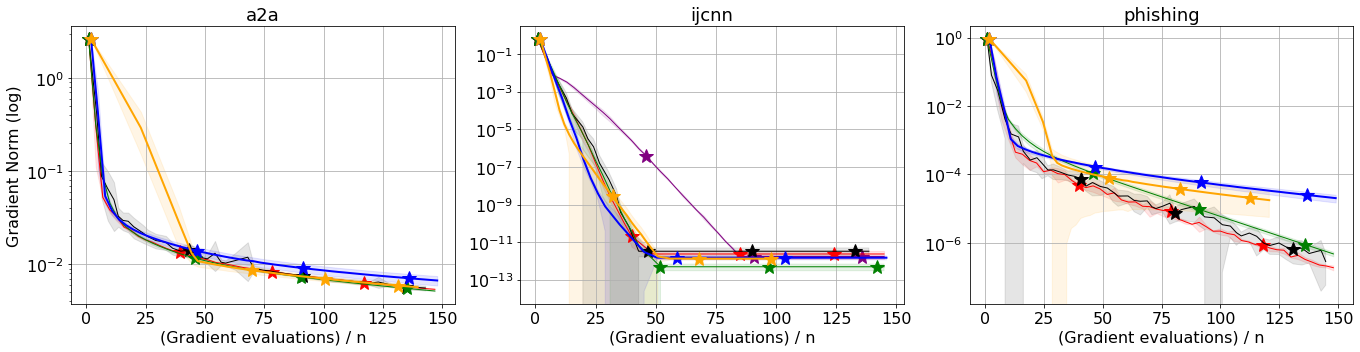}} \\
\subfigure[Sensitivity to step-size]{\label{fig:3datasets-squared-stepsize-bs=644}\includegraphics[scale = 0.3]{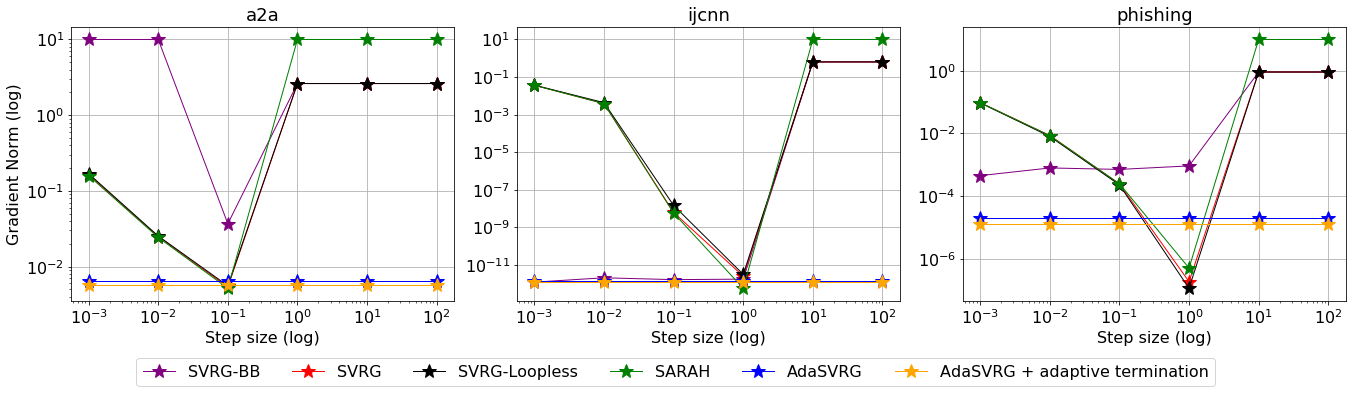}}
\caption{Comparison of AdaSVRG against \svrg/ variants, \svrg/-BB and \sarah/ for squared loss and batch-size = $64$. For the sensitivity to step-size plot to be readable, we limit the gradient norm to a maximum value of 10. In some cases, SVRG-BB diverged, and we remove the curves accordingly so as not to clutter the plots.}
\end{figure}

\clearpage
\subsection{Studying the effect of the batch-size on the performance of AdaSVRG}
\FloatBarrier
\begin{center}
\begin{figure}[H]
\centering    
\subfigure[Batch-size = 128]{\label{fig:4datasets-logistic-epoch-bs=128}\includegraphics[scale = 0.3]{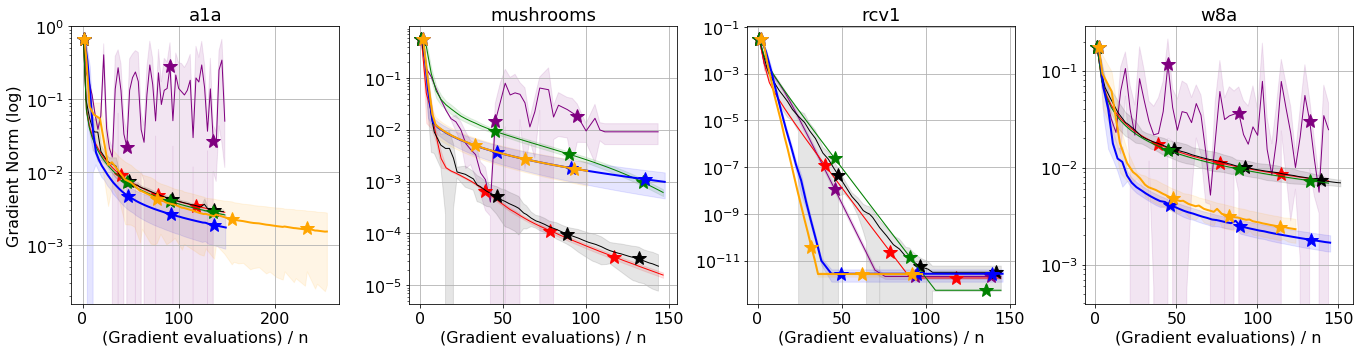}} \\
\subfigure[Batch-size = 8]{\label{fig:4datasets-logistic-epoch-bs=8}\includegraphics[scale = 0.3]{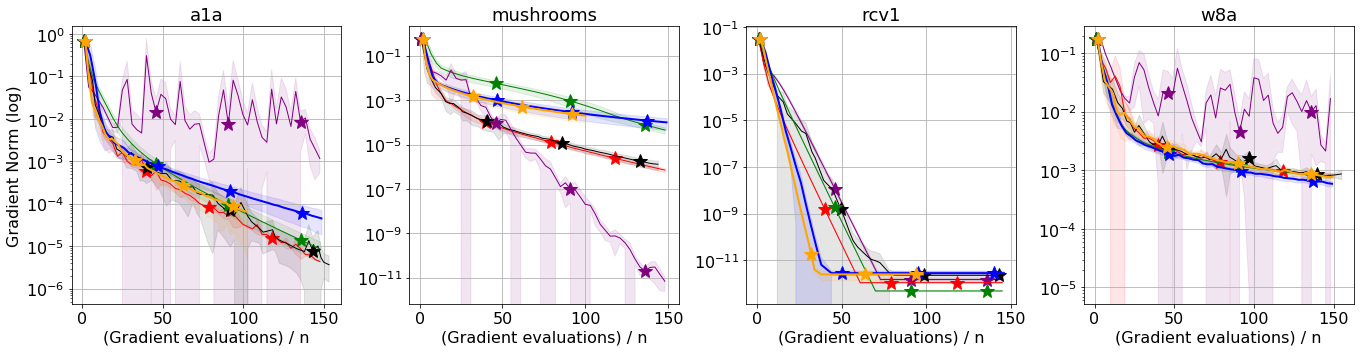}} \\
\subfigure[Batch-size = 1]{\label{fig:4datasets-logistic-epoch-bs=1}\includegraphics[scale = 0.3]{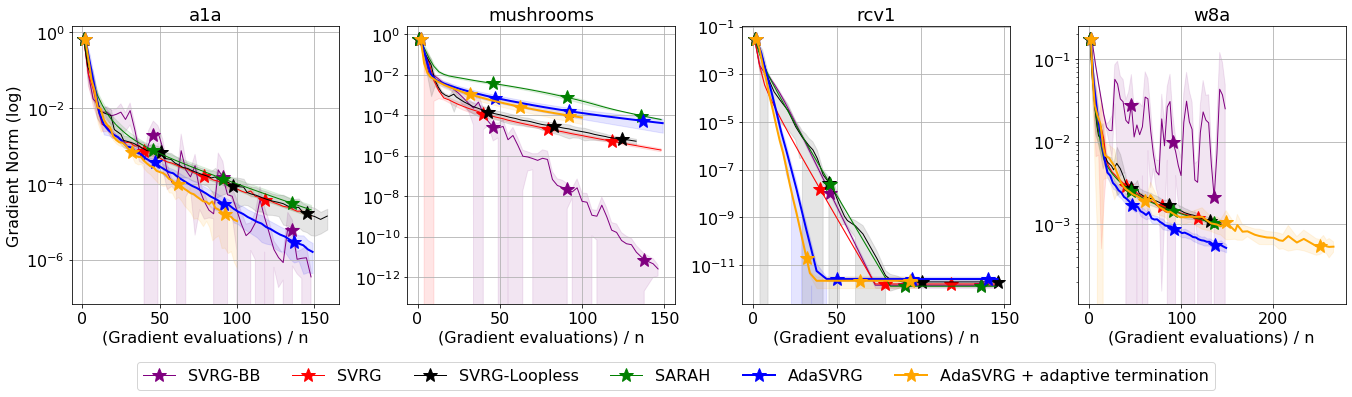}} 
\caption{Comparison of AdaSVRG against best-tuned variants of \svrg/, \svrg/-BB and \sarah/ for logistic loss for different batch-sizes.}
\end{figure}
\end{center}

\begin{figure}[H]
\centering    
\subfigure[Batch-size = 128]{\label{fig:3datasets-logistic-epoch-bs=128}\includegraphics[scale = 0.3]{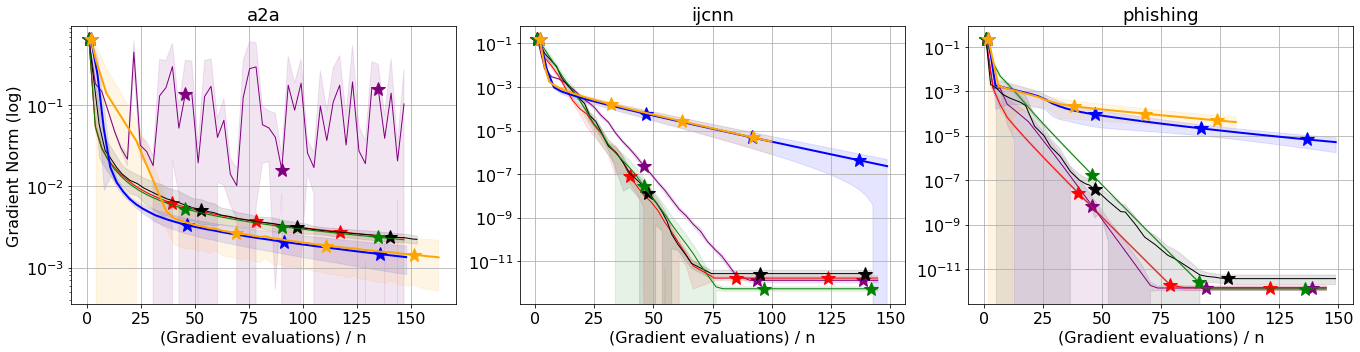}} \\
\subfigure[Batch-size = 8]{\label{fig:3datasets-logistic-epoch-bs=8}\includegraphics[scale = 0.3]{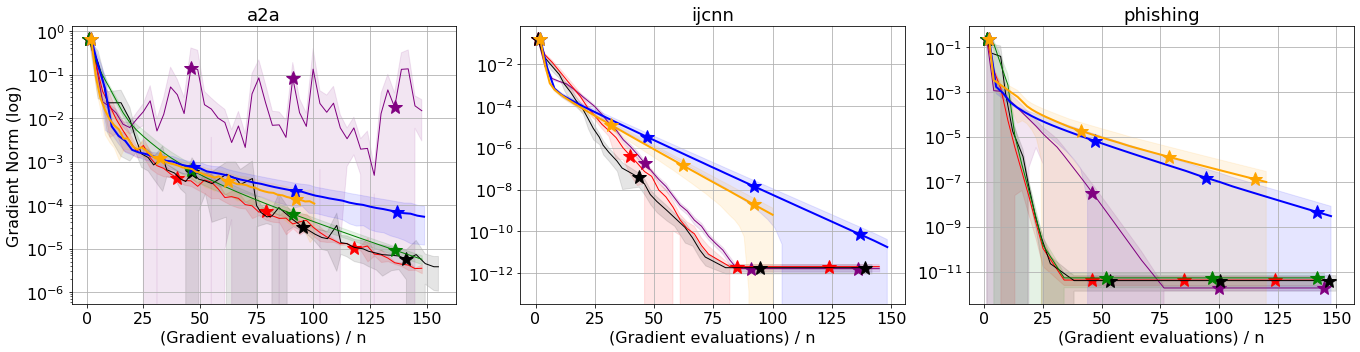}} \\
\subfigure[Batch-size = 1]{\label{fig:3datasets-logistic-epoch-bs=1}\includegraphics[scale = 0.3]{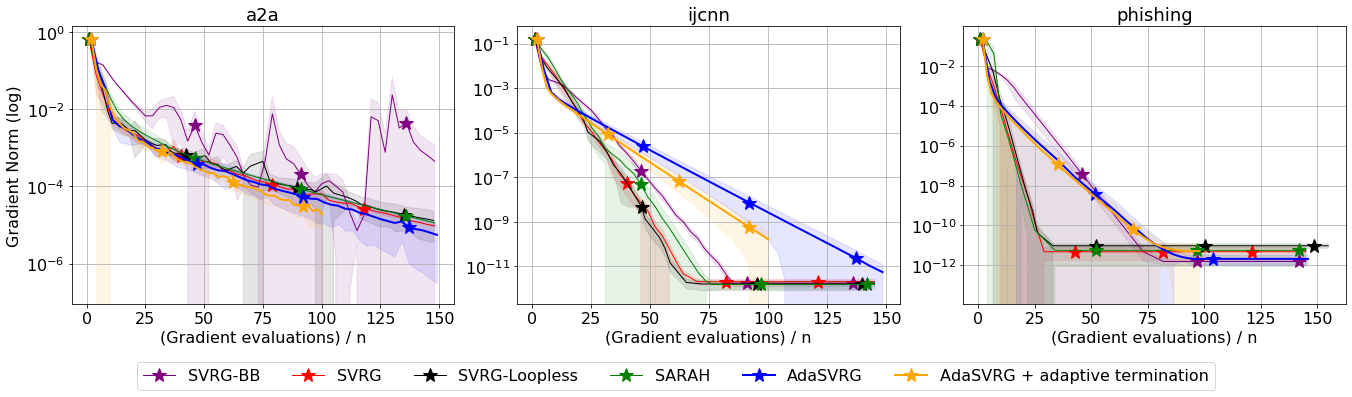}} 
\caption{Comparison of AdaSVRG against best-tuned variants of \svrg/, \svrg/-BB and \sarah/ for logistic loss for different batch-sizes.}
\end{figure}

\begin{figure}[H]
\centering    
\subfigure[Batch-size = 128]{\label{fig:4datasets-huber-epoch-bs=128}\includegraphics[scale = 0.3]{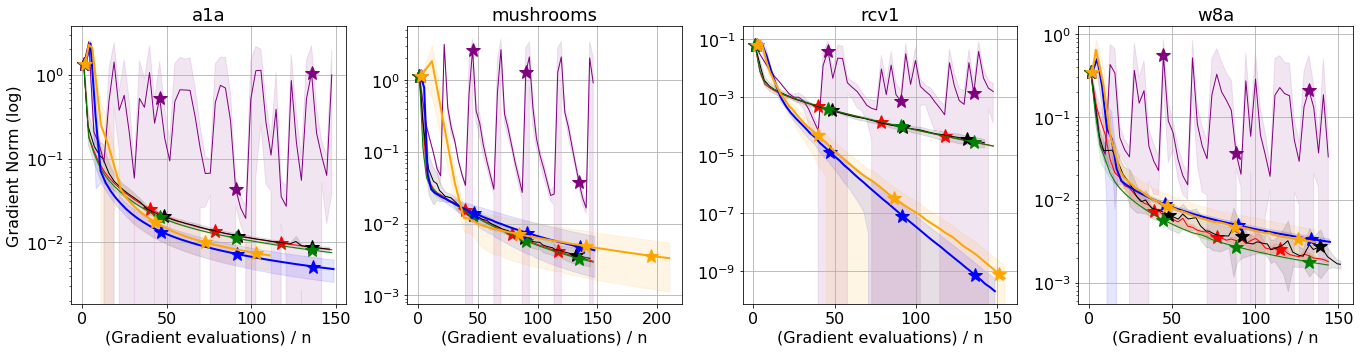}} \\
\subfigure[Batch-size = 8]{\label{fig:4datasets-huber-epoch-bs=8}\includegraphics[scale = 0.3]{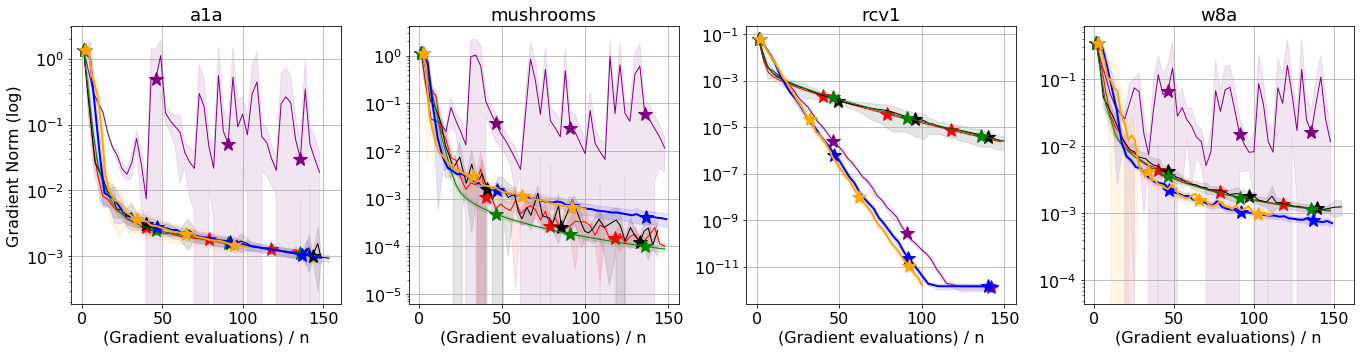}} \\
\subfigure[Batch-size = 1]{\label{fig:4datasets-huber-epoch-bs=1}\includegraphics[scale = 0.3]{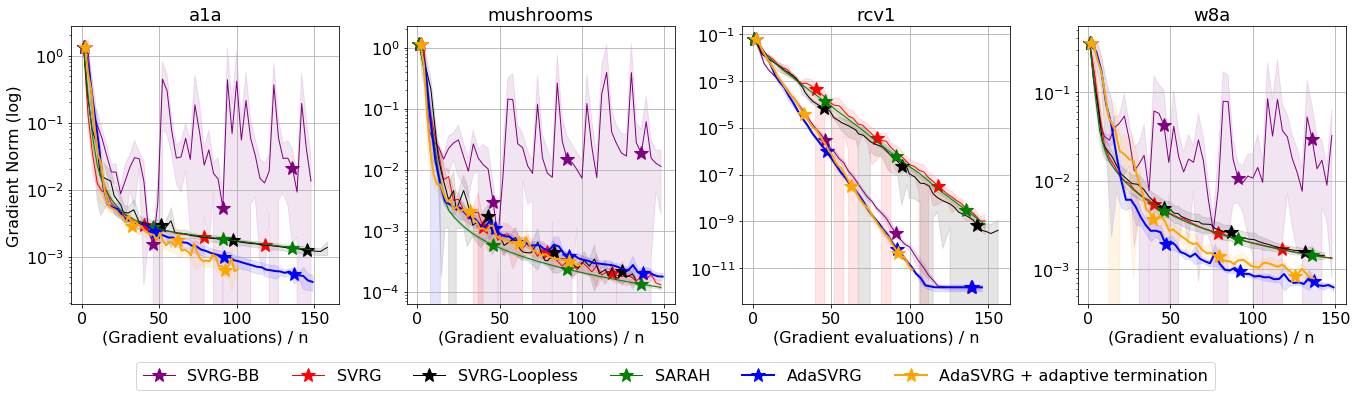}} 
\caption{Comparison of AdaSVRG against best-tuned variants of \svrg/, \svrg/-BB and \sarah/ for Huber loss for different batch-sizes.}
\end{figure}

\begin{figure}[t]
\centering    
\subfigure[Batch-size = 128]{\label{fig:3datasets-huber-epoch-bs=128}\includegraphics[scale = 0.3]{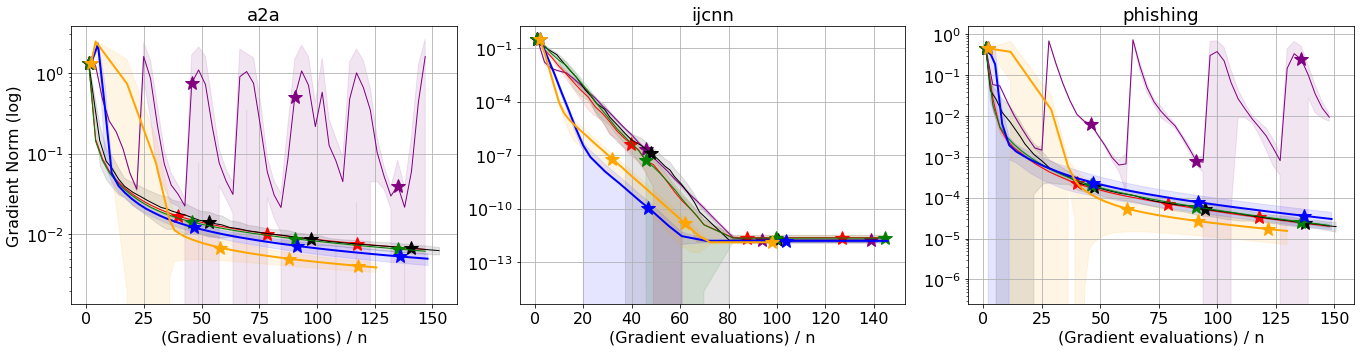}} \\
\subfigure[Batch-size = 8]{\label{fig:3datasets-huber-epoch-bs=8}\includegraphics[scale = 0.3]{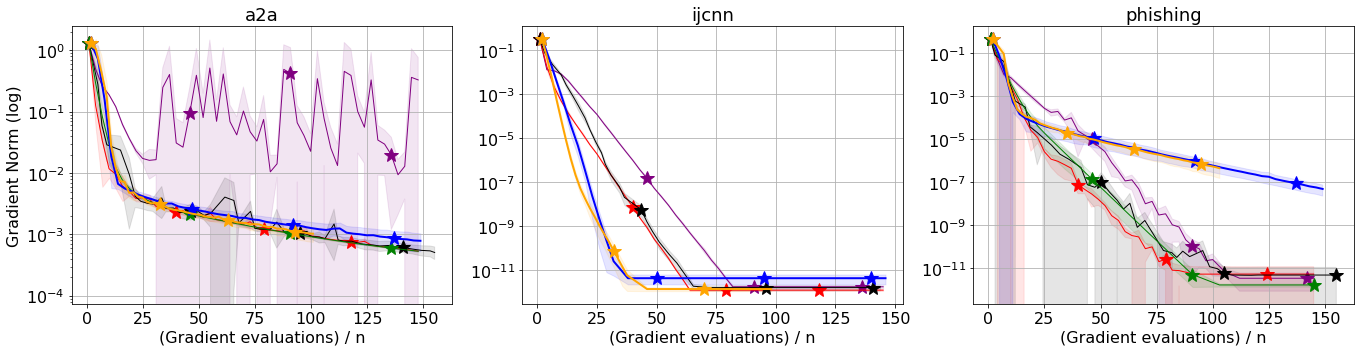}} \\
\subfigure[Batch-size = 1]{\label{fig:3datasets-huber-epoch-bs=1}\includegraphics[scale = 0.3]{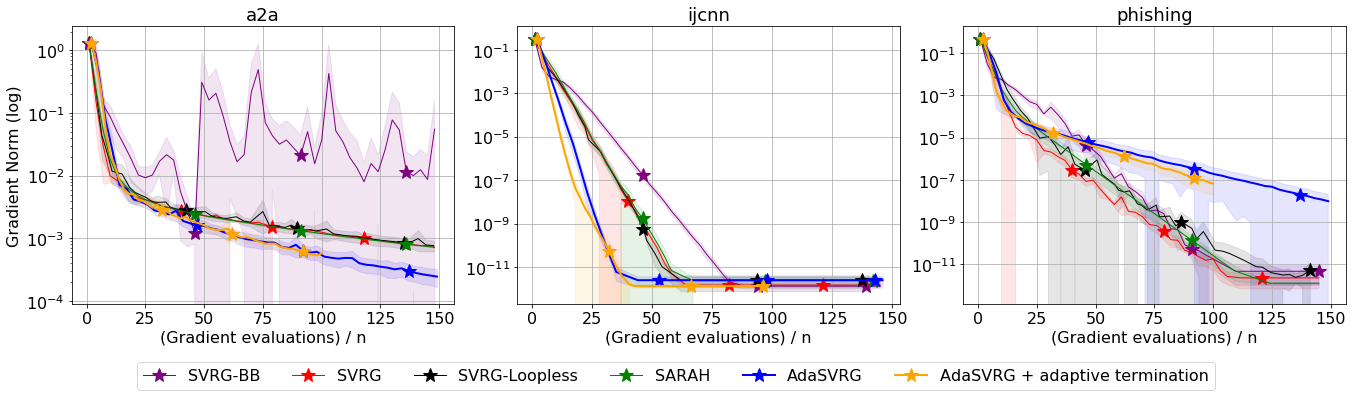}} 
\caption{Comparison of AdaSVRG against best-tuned variants of \svrg/, \svrg/-BB and \sarah/ for Huber loss for different batch-sizes.}
\end{figure}

\begin{figure}[!t]
\centering    
\subfigure[Batch-size = 128]{\label{fig:4datasets-squared-epoch-bs=128}\includegraphics[scale = 0.3]{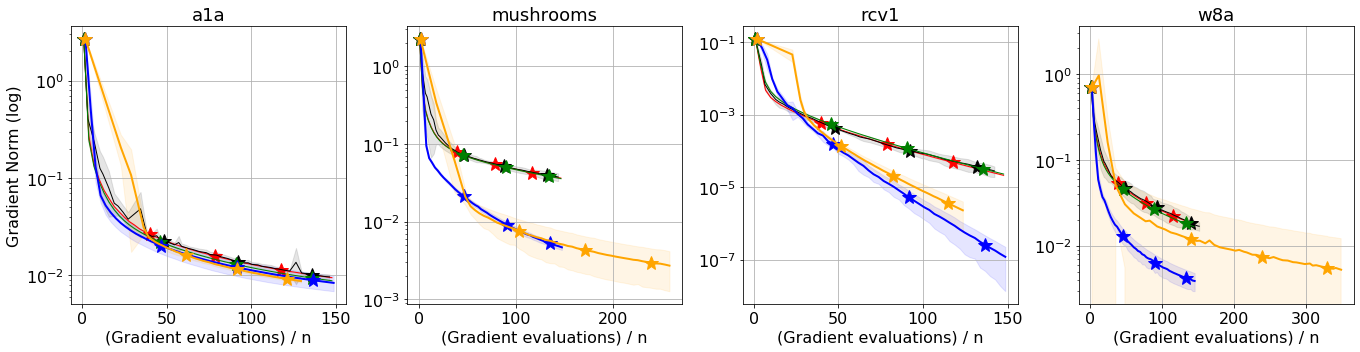}} \\
\subfigure[Batch-size = 8]{\label{fig:4datasets-squared-epoch-bs=8}\includegraphics[scale = 0.3]{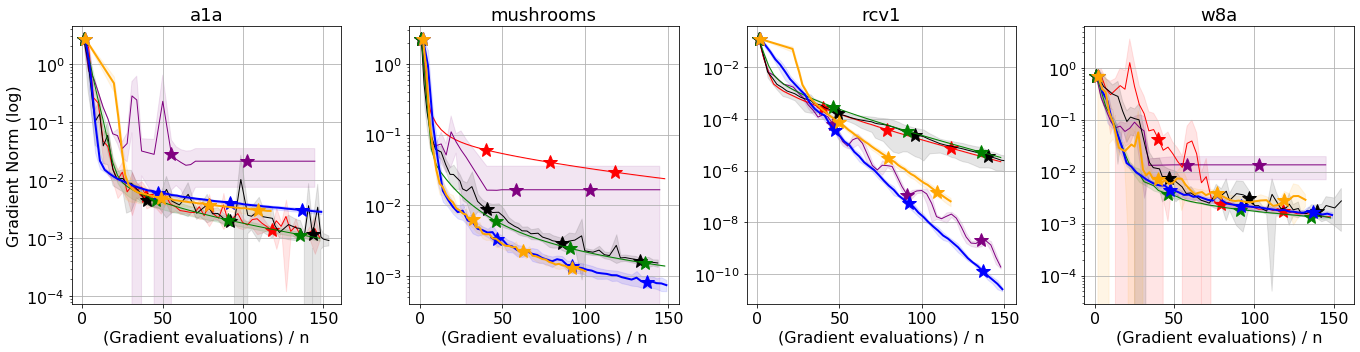}} \\
\subfigure[Batch-size = 1]{\label{fig:4datasets-squared-epoch-bs=1}\includegraphics[scale = 0.3]{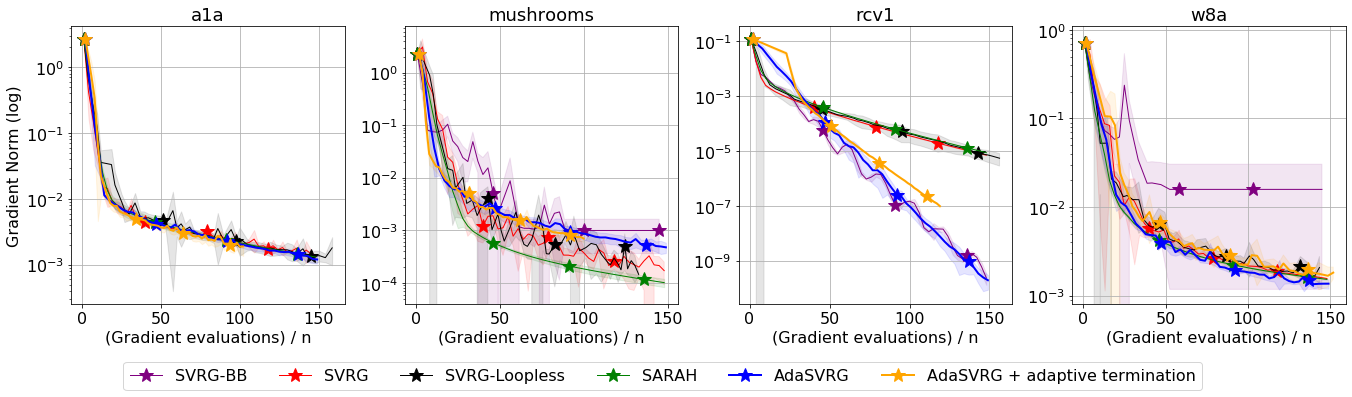}} 
\caption{Comparison of AdaSVRG against best-tuned variants of \svrg/, \svrg/-BB and \sarah/ for squared loss for different batch-sizes. In some cases, SVRG-BB diverged, and we remove the curves accordingly so as not to clutter the plots.}
\end{figure}

\begin{figure}[t]
\centering    
\subfigure[Batch-size = 128]{\label{fig:3datasets-squared-epoch-bs=128}\includegraphics[scale = 0.3]{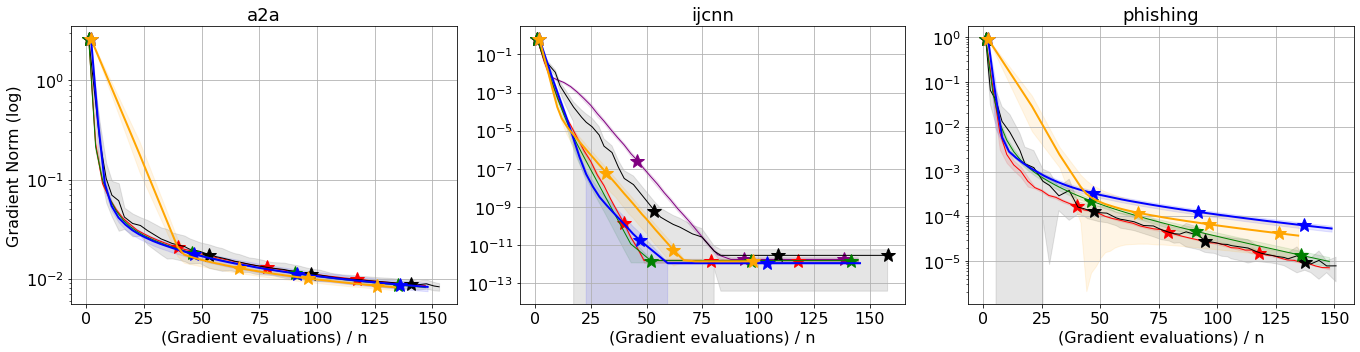}} \\
\subfigure[Batch-size = 8]{\label{fig:3datasets-squared-epoch-bs=8}\includegraphics[scale = 0.3]{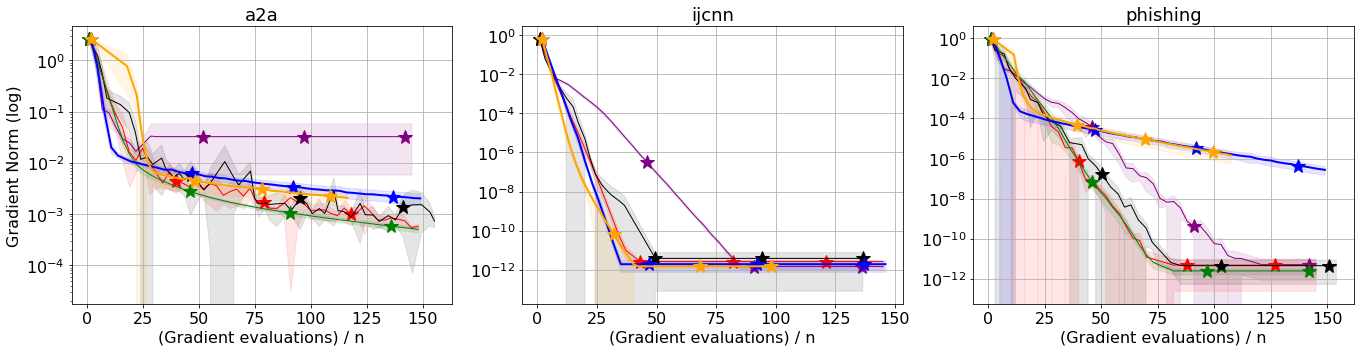}} \\
\subfigure[Batch-size = 1]{\label{fig:3datasets-squared-epoch-bs=1}\includegraphics[scale = 0.3]{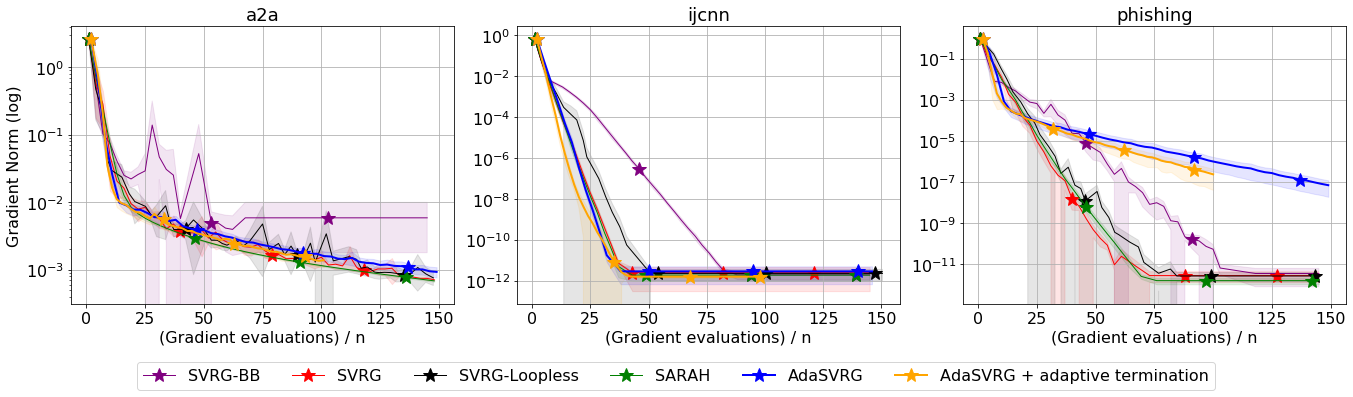}} 
\caption{Comparison of AdaSVRG against best-tuned variants of \svrg/, \svrg/-BB and \sarah/ for squared loss for different batch-sizes. In some cases, SVRG-BB diverged, and we remove the curves accordingly so as not to clutter the plots.}
\end{figure}

\clearpage
\subsection{Additional interpolation experiments}
\FloatBarrier
\begin{figure}[t]
\centering    
\subfigure[Batch-size = 128]{\label{fig:interpolation-logistic-bs=128}\includegraphics[scale = 0.3]{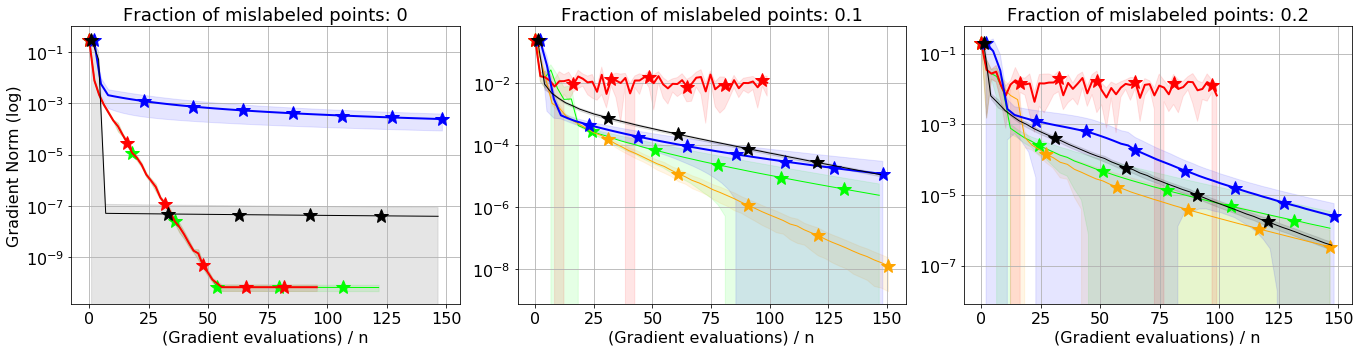}} \\
\subfigure[Batch-size = 8]{\label{fig:interpolation-logistic-bs=8}\includegraphics[scale = 0.3]{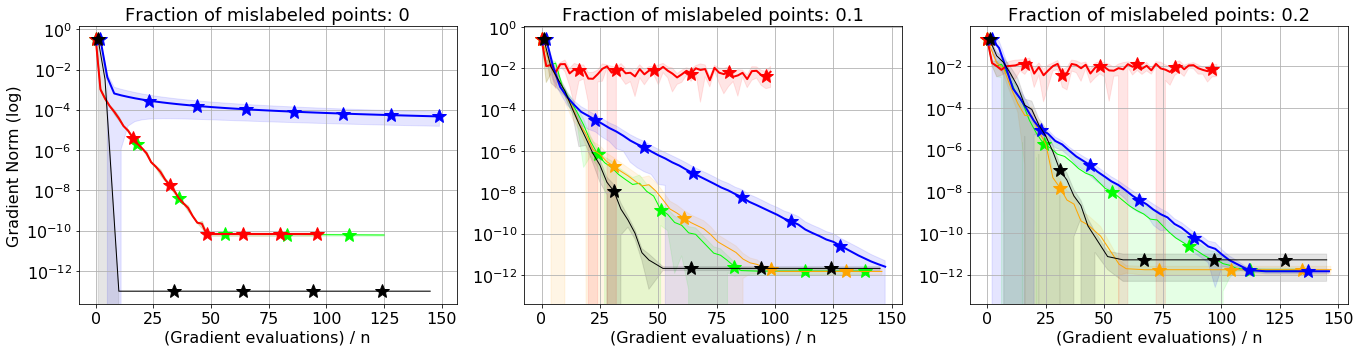}} \\
\subfigure[Batch-size = 1]{\label{fig:interpolation-logistic-bs=1}\includegraphics[scale = 0.3]{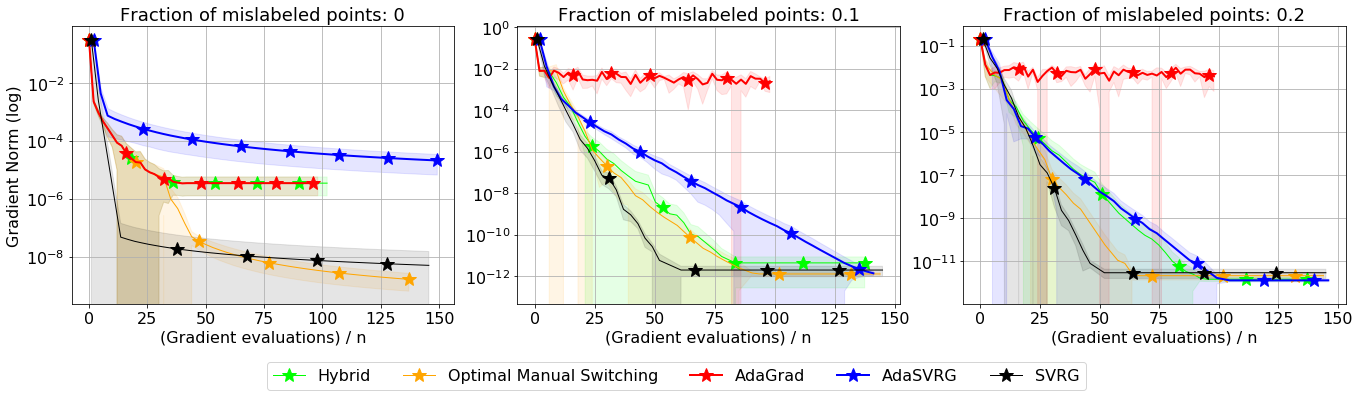}} 
\caption{Comparison of best-tuned SGD, best-tuned SVRG, \adagrad/, \adasvrg/ and~\cref{alg:hybrid} with logisitic loss on datasets with different fraction of mislabeled data-points for different batch-sizes. Interpolation is exactly satisfied for the left-most plots.}
\end{figure}

\begin{figure}[H]
\centering    
\subfigure[Batch-size = 128]{\label{fig:interpolation-hinge-bs=128}\includegraphics[scale = 0.3]{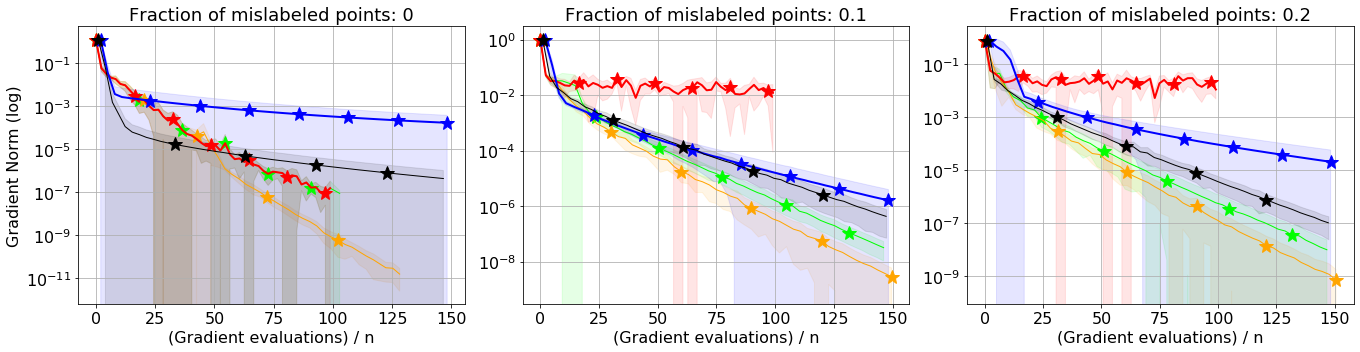}} \\
\subfigure[Batch-size = 64]{\label{fig:interpolation-hinge-bs=64}\includegraphics[scale = 0.3]{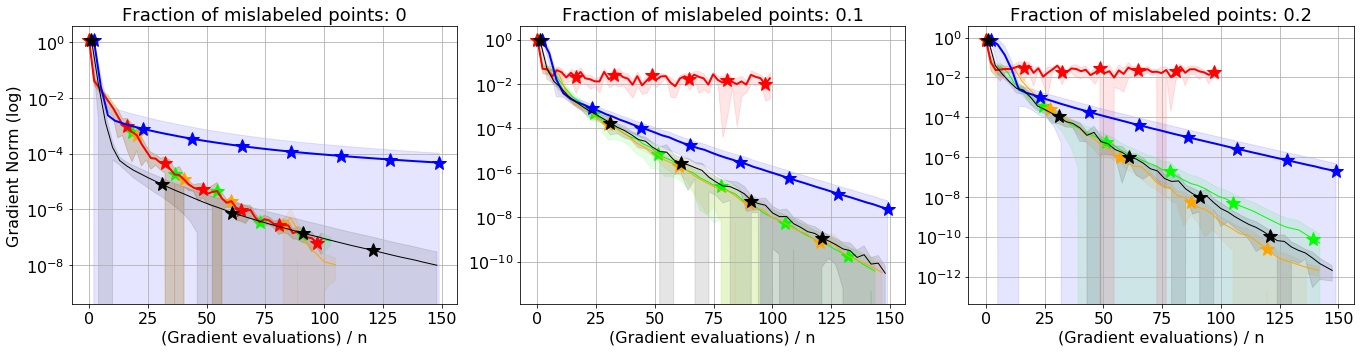}} \\
\subfigure[Batch-size = 8]{\label{fig:interpolation-hinge-bs=8}\includegraphics[scale = 0.3]{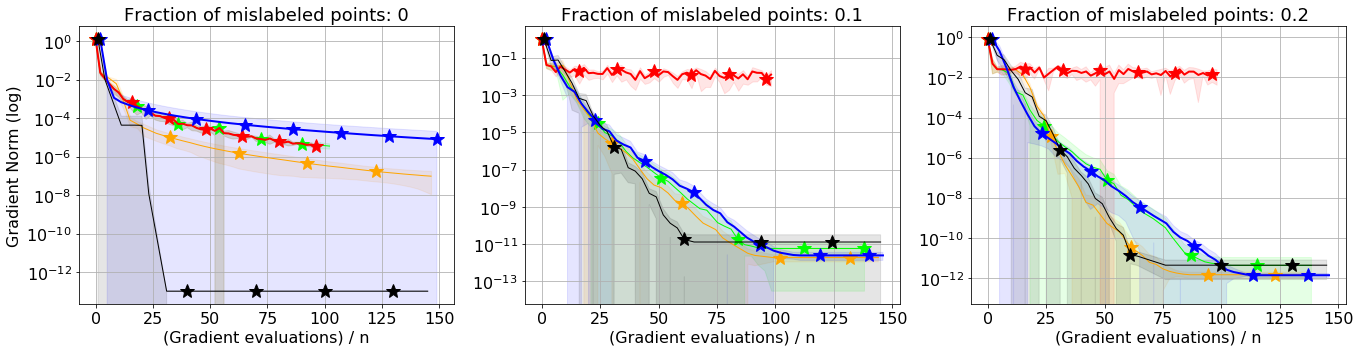}} \\
\subfigure[Batch-size = 1]{\label{fig:interpolation-hinge-bs=1}\includegraphics[scale = 0.3]{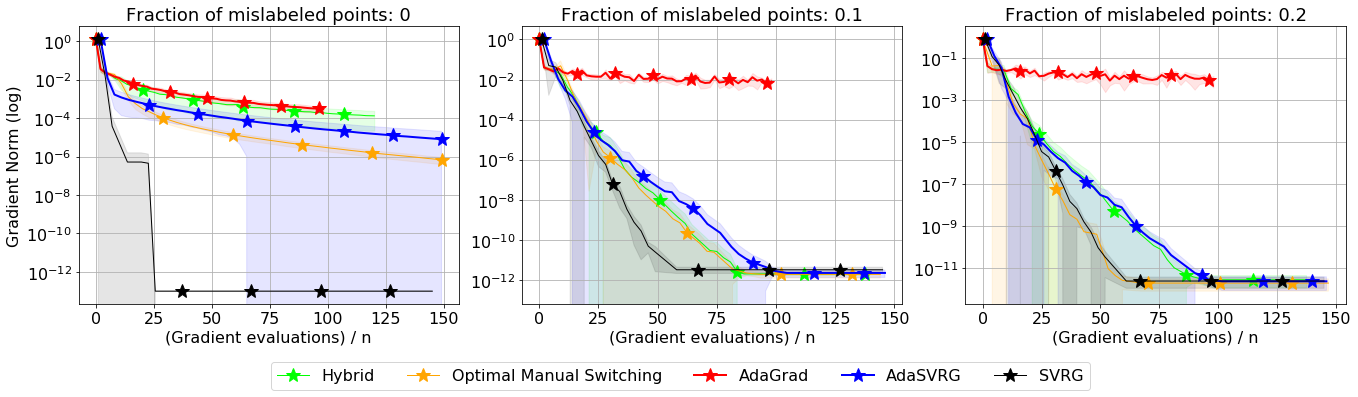}} 
\caption{Comparison of best-tuned SGD, best-tuned SVRG, \adagrad/, \adasvrg/ and~\cref{alg:hybrid} with squared hinge loss on datasets with different fraction of mislabeled data-points for different batch-sizes. Interpolation is exactly satisfied for the left-most plots.}
\end{figure}

\end{document}